\documentclass[twocolumn,letterpaper]{IEEEAerospaceCLS}  

\usepackage{amsmath,amsfonts}
\usepackage{algorithm}
\usepackage{algpseudocode}
\usepackage{array}
\usepackage[caption=false,font=normalsize,labelfont=sf,textfont=sf]{subfig}
\usepackage{textcomp}
\usepackage{stfloats}
\usepackage{url}
\usepackage{hyperref}
\usepackage{verbatim}
\usepackage{graphicx}
\usepackage{cite}
\usepackage{changepage}
\usepackage{caption}

\usepackage{lipsum}
\usepackage{bm}
\usepackage{physics}
\usepackage{mathtools}

\usepackage{enumitem}

\usepackage{amsthm}

\usepackage{xcolor}
\usepackage{siunitx}

\usepackage[english]{babel}
\newtheoremstyle{colon}%
{2mm}
{}
{\itshape}
{}
{\bfseries}
{:}
{ }
{}
\theoremstyle{colon}
\newtheorem{theorem}{Theorem}
\newtheorem{definition}{Definition}
\newtheorem{lemma}{Lemma}
\newtheorem{problem}{Problem}
\newtheorem{assumption}{Assumption}
\newtheorem{algo}{Algorithm}
\newtheorem{corollary}{Corollary}

\theoremstyle{remark}
\newtheorem{remark}{Remark}

\newcommand{\expnumber}[2]{{#1}\mathrm{e}{#2}}
\newcommand{\mbb}[1]{{\mathbb{#1}}}
\newcommand{\mcl}[1]{{\mathcal{#1}}}
\newcommand{\trsp}{{\mathsf{T}}}
\newcommand{\sothree}{{\mathrm{SO}(3)}}

\DeclareMathOperator{\bi}{\textup{\textbf{i}}}
\DeclareMathOperator{\bj}{\textup{\textbf{j}}}
\DeclareMathOperator{\bz}{\textup{\textbf{k}}}

\newcommand{\irchi}[2]{\raisebox{\depth}{$#1\chi$}}
\DeclareRobustCommand{\rchi}{{\mathpalette\irchi\relax}}
\newcommand{\AiA}[2]{\mcl{A}_{#1}\bm{i}\mcl{A}_{#2}^*}
\newcommand{\AstA}[2]{\mcl{A}_{#1}\star{A}_{#2}}
\newcommand{\Asq}[1]{\mcl{A}_{#1}^{2\star}}

\newcommand{\inred}[1]{{\leavevmode\color{red}#1}}
\newcommand{\inorange}[1]{{\leavevmode\color{orange}#1}}

\makeatletter 
\newcommand\Label[1]{&\refstepcounter{equation}(\theequation)\ltx@label{#1}&}
\makeatother

\makeatletter
\let\@oldmaketitle\@maketitle
\renewcommand{\@maketitle}{\@oldmaketitle
	\centering
	\includegraphics[width=\textwidth]{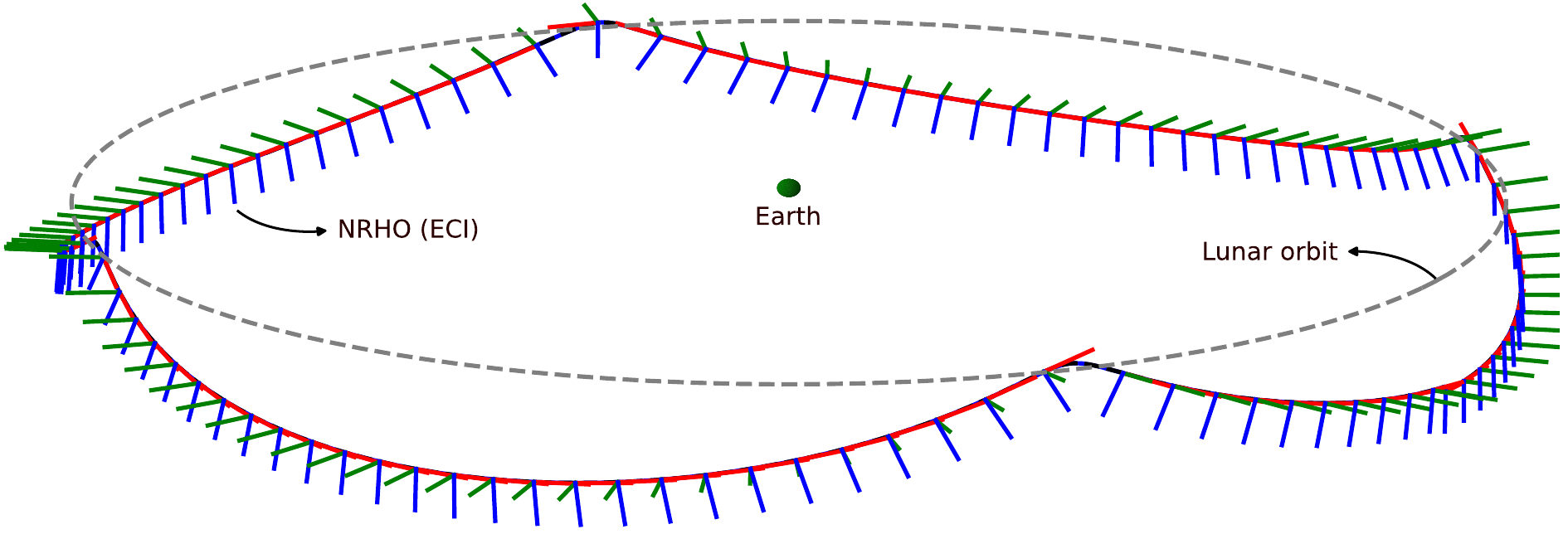}
	\captionof{figure}{ An illustrative example of the PHODCOS coordinate system applied to a highly nonlinear curve, specifically the Near Rectilinear Halo Orbit (NRHO) selected for the Lunar Gateway. The frame components are colored red, blue, and green, respectively. The Earth is represented by the green sphere in the middle, whereas the Moon's orbit is given by the dashed gray line. This depiction is presented in the Earth-Centered Inertial (ECI) frame.} 
	\label{fig:nrho_cover}
	\vspace{1mm}
	}
\makeatother

\begin{document}
\title{PHODCOS:\\Pythagorean Hodograph-based \\Differentiable Coordinate System}

\author{%
Jon Arrizabalaga\\ 
Technical University of Munich\\
Carnegie Mellon University\\
Pittsburgh, PA 15213\\
jon.arrizabalaga@tum.de
\and
Fausto Vega\\
Robotics Institute\\
Carnegie Mellon University\\
Pittsburgh, PA 15213\\
fvega@andrew.cmu.edu
\and 
 Zbyněk Šír\\
Faculty of Mathematics and Physic\\
Charles University\\
Prague, 186 75\\
zbynek.sir@mff.cuni.cz
\and
Zachary Manchester\\
Robotics Institute\\
Carnegie Mellon University\\
Pittsburgh, PA 15213\\
zacm@cmu.edu
\and
Markus Ryll\\
Dept. of Aerospace and Geodesy\\
Technical University of Munich\\
Munich, 80333\\
markus.ryll@tum.de
\thanks{\footnotesize 979-8-3503-5597-0/25/$\$31.00$ \copyright2025 IEEE}              
}

\maketitle

\thispagestyle{plain}
\pagestyle{plain}

\maketitle

\thispagestyle{plain}
\pagestyle{plain}

\begin{abstract}
This paper presents PHODCOS, an algorithm that assigns a moving coordinate system to a given curve. The parametric functions underlying the coordinate system, i.e., the path function, the moving frame and its angular velocity, are exact -- approximation free -- differentiable, and sufficiently continuous. This allows for computing a coordinate system for highly nonlinear curves, while remaining compliant with autonomous navigation algorithms that require first and second order gradient information. In addition, the coordinate system obtained by PHODCOS is fully defined by a finite number of coefficients, which may then be used to compute additional geometric properties of the curve, such as arc-length, curvature, torsion, etc.  Therefore, PHODCOS presents an appealing paradigm to enhance the geometrical awareness of existing guidance and navigation on-orbit spacecraft maneuvers. The PHODCOS algorithm is presented alongside an analysis of its error and approximation order, and thus, it is guaranteed that the obtained coordinate system matches the given curve within a desired tolerance. To demonstrate the applicability of the coordinate system resulting from PHODCOS, we present numerical examples in the Near Rectilinear Halo Orbit (NRHO) for the Lunar Gateway.
\end{abstract} 
\vspace{-5mm}
\begin{flushleft}
\textbf{Code}: {\small \url{https://github.com/jonarriza96/phodcos}}
\end{flushleft}

\tableofcontents

\newpage
\section{Introduction}
Understanding and characterizing the intricate shapes and motions present in nature is a fundamental goal in both science and engineering. Although the familiar Cartesian representation excels at depicting simple lines and circles, it frequently falls short in capturing the complexities of curves and surfaces seen in the real world. Consequently, developing more adaptable tools that can represent the spatiotemporal aspects of physical phenomena has been crucial in the study of dynamic systems and their control algorithms.

Within this scope, modeling the behavior of a dynamical system relies on three design choices: 1) the \emph{states} that parameterize the system, 2) an \emph{evolution variable} upon which the system progresses, e.g. time or space, and 3) a \emph{reference} state that quantifies the motions of the system in a relative fashion. In the specific case where the states describe the motion of a body's position and/or orientation, this reference becomes a \emph{reference frame}, which can be stationary or moving. The optimal choice of these three ingredients is highly dependent on the particular system and its intended application.

Restricting ourselves to body motions, of all three ingredients, the first and second are specific to the system. Nevertheless, the third ingredient --the selection of a reference frame-- is universal, as it remains independent of the governing equations of motion. For this reason, we propose a method to assign a \emph{differentiable moving frame} to any given curve. The suggested algorithm shows attractive attributes to be used as a reference frame for dynamical systems, as well as for characterizing complex shapes and their geometrical properties.

In particular, this paper presents \textbf{PHODCOS}: a \textbf{P}ythagorean \textbf{HO}dograph-based \textbf{D}ifferentiable \textbf{CO}ordinate \textbf{S}ystem. The moving frame resulting from this algorithm is characterized by the following four properties: First, it is built upon \emph{Pythagorean-Hodograph (PH) curves}, a class of polynomials that efficiently parameterize geometric features like arc-length, curvature, torsion, moving frame components and its angular velocity. Second, the algorithm operates using closed-form analytical expressions, avoiding numerical or iterative processes, making it \emph{computationally efficient}. Third, the conversion \emph{error is bounded}, ensuring that the coordinate system aligns with the curve within a defined tolerance. Fourth, the moving frame and angular velocity are \emph{twice differentiable}, making the coordinate system compliant with algorithms based on learning or optimization. To achieve this, our method relies on four main contributions:
\begin{enumerate}[leftmargin=2em]
    \item We extend the $C^2$ hermite interpolation method presented in~\cite{vsir2007} to $C^4$ and introduce a closure mechanism, thereby, ensuring that the coordinate system is twice differentiable.
    \item We conduct an approximation order analysis and provide a theoretical proof that guarantees the coordinate system to be within a desired tolerance of the curve.
    \item We showcase the method's applicability in a highly nonlinear exemplary curve: the Near Rectilinear Halo Orbit (NRHO) for the Lunar Gateway.
    \item We \href{https://github.com/jonarriza96/phodcos}{open source} the implementation of the algorithm, alongside the numerical analysis and examples shown in this paper.
\end{enumerate}

The remainder of this manuscript is structured as follows: Section~\ref{sec:related_work} provides an overview of existing methods for assigning a moving frame to a given curve. Section~\ref{sec:problem_statement} formally defines the problem solved in this paper.  After introducing the required preliminaries in Section~\ref{sec:preliminaries}, Section~\ref{sec:phodcos} presents the PHODCOS algorithm, by providing its derivation and a numerical analysis. Section~\ref{sec:example} shows an illustrative example by testing PHODCOS in the Lunar Gateway before Section~\ref{sec:conclusions} presents the conclusions.

\section{How to frame a curve?}~\label{sec:related_work}

The assignment of a moving frame to a curve is a well-studied problem, with several solutions available in the existing literature. The most commonly used are the Frenet Serret Frame (FSF)~\cite{struik1961lectures,abbena2017modern} and the Parallel Transport Frame (PTF)~\cite{bishop1975there,hanson1995parallel,wang2008computation}. From a computational perspective, the FSF is purely analytical, as it is derived from local derivative information in closed form. In contrast, the PTF depends on global information and requires a numerical method for its computation.

The FSF is frequently chosen in the literature because of its straightforward analytical nature. However, it fails when the reference path is straight (i.e., zero curvature) and introduces an unnecessary twist to its first component. The PTF is the most recognized alternative, as it avoids both singularities and twists. Yet, current techniques for calculating PTFs~\cite{wang2008computation} only focus on the rotation matrix and neglect the angular velocity. Moreover, these methods are discrete, preventing the computation of higher-order derivatives. In the age of data-driven methods, prominent decision-making algorithms increasingly rely on learning and optimization~\cite{malyuta2021convex,kaufmann2023champion}, making this constraint particularly noteworthy.


To address these challenges, PHODCOS opts for a lesser-known yet promising alternative, the Euler Rodrigues Frame (ERF)~\cite{choi2002euler}. There are two main reasons for the ERF's reduced visibility compared to the more established counterparts: First, because they depend on PH curves and their algebra, extracting an ERF from any arbitrary curve while keeping a bounded approximation error is a daunting task~\cite{farouki2008pythagorean}.  Broadly speaking, existing methods either rely on numerical and iterative methods, resulting in computationally expensive algorithms without any guarantees on the approximation error~\cite{albrecht2020spatial} or are given in closed-form, but require certain assumptions over the given curve, limiting their universality~\cite{farouki2012design}. The second reason also relates to the ERF's inherent complex algebra, as it significantly raises the entry barrier, thus hampering its use in other scientific or engineering contexts~\cite{otto2021geometric,arrizabalaga2022spatial,arrizabalaga2023sctomp,arrizabalaga2024differentiable}.

In this work, we introduce a technique that addresses both major limitations. For the first drawback, which involves the universal and efficient computation of an ERF frame, we utilize a hermite interpolation method similar to that in~\cite{vsir2007}. This approach is expressed in closed form, eliminating the need for iterative or optimization procedures, and can be applied to any $C^4$ curve, a requirement that is easily met. Concerning the second limitation related to usability, we offer an open-source implementation of our technique, equipping practitioners with a practical toolkit and facilitating its application without requiring a deep understanding of the underlying mathematical concepts.




\section{Problem Statement}\label{sec:problem_statement}
\noindent Let $\Gamma$ refer to a geometric reference and be defined as a curve whose position is given by a function $\bm{\gamma}\,:\,\mbb{R}\rightarrow \mbb{R}^3$ that depends on path parameter $\xi$
\begin{equation}\label{eq:gamma_path}
  \Gamma = \{\xi\in\left[\xi_0,\xi_f\right] \subseteq \mbb{R}\rightarrow\bm{\gamma}(\xi) \in \mbb{R}^{3}\}\,,
\end{equation}
which is assumed to be of differentiability class $C^m$, i.e., $\bm{\gamma}\in C^m(\xi)$, where $m\geq2$.

\begin{problem}\label{problem:path_parameterization}
  Given path $\Gamma$ in~\eqref{eq:gamma_path}, design an algorithm that approximates it by 
  \begin{itemize}[leftmargin=2em]
    \item[\textbf{P1}] finding a parametric path-function $\bm{p}\,:\,\mbb{R}\rightarrow \mbb{R}^{3}$ that guarantees the approximation error to remain within a desired tolerance $\epsilon$:
    \begin{equation}
      E = \max_{\xi\in\left[\xi_0,\xi_f\right]}{||\bm{\gamma}(\xi)-\bm{p}(\xi)||} \leq \epsilon
    \end{equation}
    \item[\textbf{P2}] associating an adapted-frame\footnote[4]{Adapted-frame: A moving rotation matrix $\text{R}=\left[\bm{e_1},\bm{e_2},\bm{e_3}\right]$ associated to path $\Gamma$ whose first component $\bm{e_1}$ coincides with the tangent of the path $\dv{\bm{\gamma}(\xi)}{\xi}$.}, whose components are given by a parametric frame-function $\text{R}\,:\,\mbb{R}\rightarrow\mbb{R}^{3x3}$, and its angular velocity is defined by a parametric omega-function $\bm{\omega}\,:\,\mbb{R}\rightarrow\mbb{R}^{3}$,
    \item[\textbf{P3}] maintaining a minimum differentiability class of $C^2$ for all parametric-functions, i.e., $\{\bm{p},\text{R},\bm{\omega}\}\in C^m(\xi)$ with $m\geq2$.
  \end{itemize}
\end{problem}

\section{Preliminaries}\label{sec:preliminaries}
\noindent The algorithm proposed for the presented in Section~\ref{sec:problem_statement} leverages Pythagorean Hodograph (PH) curves, a subset of polynomials whose algebra relies on quaternions. To account for this, in this section we outline the required mathematical notation and background.
\subsection{Quaternions}
\noindent Quaternion algebra considers a four-dimensional vector space and is expressed in the form
\begin{equation}\label{eq:quaternion}
\mcl{A} = a + a_x\bi + a_y\bj + a_z\bz\,,
\end{equation}
where $\{a,a_x,a_y,a_z\}\in\mbb{R}$ and the canonical basis $\{\bm{1} = (1,0,0,0),\,\bi = (0,1,0,0),\, \bj = (0,0,1,0),\, \bz = (0,0,0,1)\}$ fulfills the following multiplication rules:
\begin{gather*}
    \bi^2 = \bj^2 = \bz^2 = \bi\bj\bz = -1\,,\\
    \bi\bj = -\bj\bi = \bz,\,\,\bj\bz = -\bz\bj = \bi,\,\,\bz\bi = -\bi\bz =\bj\,.
\end{gather*}
For a more detailed introduction, please refer to \cite{kuipers1999quaternions}. The conjugate of quaternion \eqref{eq:quaternion} is $\mcl{A}^* = a - a_x\bi - a_y\bj - a_z\bm{z} $ and its absolute value is defined as
\begin{equation*}
|\mcl{A}| = \sqrt{\mcl{A}\mcl{A^*}} = \sqrt{\mcl{A^*}\mcl{A}} = \sqrt{a^2 + a_x^2 + a_y^2 + a_z^2}\,.
\end{equation*}
Unit quaternions,  characterized by $|A|=1$, are a singularity-free and minimal representation of $\sothree$. We refer to unit quaternions with vanishing $\bj$ and $\bz$ components as
\begin{equation*}
    Q_\phi = \cos{\phi}+\bi\sin{\phi}\,.
\end{equation*}
Pure vector quaternions ascribe to quaternions whose real part vanishes, i.e., $a=0$. In the sequel we will leverage the following definitions and lemmas.

\begin{definition}[\bfseries Commutative operation]\label{definition:comm_operation}
The commutative multiplication on the quaternions $\mcl{A}$ and $\mcl{B}$ is defined as
\begin{equation}
\mcl{A}\star\mcl{B}\coloneqq\frac{1}{2}\left(\mcl{A}\bi\mcl{B}^* + \mcl{B}\bi\mcl{A}^*\right)
\end{equation}
In the case of two equivalent quaternions, the respective $n$-th power is represented as $\mcl{A}^{n\star} = \underbrace{\mcl{A}\star\mcl{A}\star\cdots\star\mcl{A}}_{n}$ .
\begin{remark}
$\mcl{A}\star \mcl{B}$ is equal to the vector part of $\mcl{A}\bi\mcl{B}$, and thus is a pure quaternion. 
\end{remark}
\end{definition}
Later on, we will observe that the solutions to the quaternion equivalents of quadratic and linear equations play a crucial role in the development of PH hermite interpolants. In the subsequent lemmas, we provide a detailed description of these solutions.

\begin{lemma}\label{lemma:AB}
Given a pure quaternion $\bm{a}$ and a non-zero quaternion $\mcl{B}$, the set of all solutions to the linear equation
\begin{subequations}\label{eq:AiBa}
\begin{equation}
    \mcl{X}\star\mcl{B} = \bm{a}
\end{equation}
form the one-parameter family
\begin{equation}
    \mcl{X}_\tau = -\frac{(\tau+\bm{a})\mcl{B}\bi}{|\mcl{B}|^2}\quad\text{with}\,\, \tau\in\mbb{R}\,.
\end{equation}
\end{subequations}
\end{lemma}
\begin{lemma}\label{lemma:A2}
Given that $\bm{a}$ is a specific pure quaternion and not a negative multiple of $\bi$, then every solution to the quadratic equation
\begin{subequations}\label{eq:AiAa}
\begin{equation}
    \mcl{X}^{2\star} = \bm{a}
\end{equation}
form the one-parameter family
\begin{equation}
    \mcl{X}_\phi = \sqrt{|\bm{a}|}\frac{\frac{\bm{a}}{|\bm{a}|}+\bi}{|\frac{\bm{a}}{|\bm{a}|}+\bi|}\mcl{Q}(\phi) \quad\text{with}\,\,\phi\in[0,2\pi)\,.
\end{equation}
\end{subequations}
\end{lemma}
\subsection{Pythagorean Hodograph (PH) curves}
\noindent PH curves are a subset of polynomial curves whose parametric speed is a polynomial of the path parameter $\xi$ \cite{farouki2008pythagorean}. Decomposing the components of the path-function as $\bm{p}(\xi) = \left[x(\xi),y(\xi),z(\xi)\right]^\trsp$, the condition for a polynomial to be a PH curve is equivalent to
\begin{gather}\label{eq:ph_cond}
    \sigma^2(\xi) = {x'}^2(\xi)+ {y'}^2(\xi) + {z'}^2(\xi)\,,
\end{gather}
where $\sigma(\xi)$ is a polynomial. More intuitively, as its name suggests, the curve's hodograph $\bm{h}(\xi) = \dv{\bm{p}(\xi)}{\xi}=  \left[x'(\xi),y'(\xi),z'(\xi)\right]^\trsp$, i.e., derivative with respect to path parameter $\xi$, satisfies a Pythagorean condition.
As stated in \cite{dietz1993algebraic}, \eqref{eq:ph_cond} only holds true if  $\{u,v,p,q\}\,:\,\mbb{R}\rightarrow\mbb{R}$ polynomials on path parameter $\xi$ exist, such that
\begin{subequations}\label{eq:ph_cond2}
\begin{align}
x'(\xi) &= u^2(\xi) + v^2(\xi) - p^2(\xi) - q^2(\xi)\,,\\
y'(\xi) &= 2u(\xi)q(\xi) - 2v(\xi)p(\xi)\,,\\
z'(\xi) &= 2v(\xi)q(\xi) - 2u(\xi)p(\xi)\,,\\
\sigma(\xi) &=  u^2(\xi) + v^2(\xi) + p^2(\xi) + q^2(\xi)\,.
\end{align}
\end{subequations}
Given the four-dimensional structure of \eqref{eq:ph_cond2}, combining it with condition~\eqref{eq:ph_cond} and converting it to quaternion algebra allows for a compact definition of the PH condition \cite{dietz1993algebraic,farouki2002structural}. This leads to the following description of PH curves.
\begin{definition}[\bfseries PH curves]
A space polynomial curve $\bm{p}(\xi) = x(\xi)\bi+y(\xi)\bj+z(\xi)\bz$ is a PH curve if and only if there exists a quaternion polynomial $\mcl{A}(\xi) = u(\xi)+v(\xi)\bi+p(\xi)\bj+q(\xi)\bz$ such that the hodograph fulfills
\begin{equation}\label{eq:hodograph}
    \bm{h}(\xi) = \mcl{A}(\xi)\bi\mcl{A}^*(\xi) = \mcl{A}^{2\star}(\xi)\,.
\end{equation}
\end{definition}

\noindent From this expression it follows that a quaternion polynomial $\mcl{A}(\xi)$ of degree $n$ is related to a PH curve whose hodograph $\bm{h}(\xi)$ and path-function $\bm{p}(\xi)$ are of degree $2n$ and $2n+1$, respectively. Given their polynomial nature, they can be written in the Bernstein-Bezier form as
\begin{subequations}\label{eq:polynomial_funcs}
\begin{align}
\mcl{A}(\xi) &= \sum_{i=0}^{n}\mcl{A}_i B_i^n(\xi)\,,\label{eq:A_func}\\
\bm{h}(\xi) &= \sum_{i=0}^{2n}\bm{h}_i B_i^{2n}(\xi)\,,\label{eq:h_func}\\
\bm{p}(\xi) &= \sum_{i=0}^{2n+1}\bm{p}_i B_i^{2n+1}(\xi)=\bm{p}_0+\frac{\sum_{i=0}^{2n}\bm{h}_i B_i^{2n}(\xi)}{2n+1}\,,\label{eq:path_func}
\end{align}
where $\mcl{A}_i$, $\bm{h}_i$ and $\bm{p}_i$ refer to $i$-th control point of the polynomial functions and $B_j^{n}(\xi)=\binom{n}{j}\xi_j\left(1-\xi\right)^{n-j}$ are the Bernstein polynomials. Notice that according to eq.~\eqref{eq:path_func}, the control-points of the hodograph $\bm{h}_{0,\dots,2n}$ relate to the ones of the path-function $\bm{p}_{1,\dots,2n+1}$. More specifically, for a given starting point $\bm{p}_0\in\mbb{R}^3$,
\begin{equation}\label{eq:p_ctrlpts}
    \bm{p}_{i} = \bm{p}_0+\frac{1}{2n+1}\sum_{j=0}^{i-1}\bm{h}_j,\quad i = 1,...,2n+1\,.
\end{equation}
\end{subequations}
Another notable attribute of PH curves is that they inherit an adapted-frame, denoted as \emph{Euler Rodrigues Frame} (ERF), which also exclusively depends on the quaternion polynomial \cite{choi2002euler}.  Its respective rotation matrix is defined as
\begin{multline}\label{eq:erf}
    \text{R}(\xi) = \left[\bm{e_1}(\xi),\bm{e_2}(\xi),\bm{e_3}(\xi)\right]=\\
    \frac{\left[\mcl{A}(\xi)\bi\mcl{A}^*(\xi), \mcl{A}(\xi)\bj\mcl{A}^*(\xi), \mcl{A}(\xi)\bz\mcl{A}^*(\xi)\right]}{|\mcl{A}(\xi)|^2}\,.
\end{multline}
and its angular velocity is given by
\begin{subequations}\label{eq:angvel}
\begin{equation}
    \bm{\omega}(\xi)=\rchi_1(\xi)\bm{e_1}+\rchi_2(\xi)\bm{e_2}+\rchi_3(\xi)\bm{e_3}\,,
\end{equation}
where
\begin{equation}\label{eq:angvel_comp}
    \rchi(\xi) = \{\bm{e_2}'(\xi)\,\bm{e_3}(\xi),\bm{e_3}'(\xi)\,\bm{e_1}(\xi),\bm{e_1}'(\xi)\,\bm{e_2}(\xi)\}\,,
\end{equation}
\end{subequations}
which it likewise is exclusively reliant on the quaternion polynomial. Besides that, given that the parametric speed $\sigma(\xi)$ is a polynomial defined by the condition~\eqref{eq:ph_cond}, PH curves also allow for the closed form calculation for other geometrically meaningful properties, such as the arc-length $L(\xi)$, curvature $\kappa(\xi)$ and torsion $\tau(\xi)$:
\begin{subequations}
\begin{align}\label{eq:L_kappa_tau}
    L(\xi)&=\int_{\xi_b}^{\xi_e} \sigma(\xi)\,d\xi\\
    \kappa(\xi) &= \frac{||\bm{p}^{'}(\xi) \cross \bm{p}^{''}(\xi)||}{||\bm{p}^{'}(\xi)||^3}\\
    \tau(\xi) &= \frac{\left[\bm{p}^{'}(\xi)\cross\bm{p}^{''}(\xi)\right]\cdot\bm{p}^{'''}(\xi)}{||\bm{p}^{'''}(\xi)||^2}
\end{align}
\end{subequations}
where $\xi_b$ and $\xi_e$ are the starting and ending values for the path parameter.

From all the parametric functions introduced in this subsection~\eqref{eq:polynomial_funcs}, ~\eqref{eq:erf}, ~\eqref{eq:angvel} and ~\eqref{eq:L_kappa_tau}, it can be stated that path parameterizing a curve as defined in Section~\ref{sec:problem_statement} reduces to constructing a PH curve by finding an appropriate quaternion polynomial $\mcl{A}(\xi)$, which from now onward we will refer to as \emph{preimage}.

\section{The algorithm: PHODCOS}\label{sec:phodcos}

\noindent To address problem~\eqref{problem:path_parameterization}, we present \textbf{PHODCOS}: a \textbf{P}ythagorean \textbf{HO}drograph-based \textbf{D}ifferentiable   \textbf{CO}ordinate \textbf{S}ystem. It leverages the appealing properties of PH curves, i.e., their compact representation via the preimage and an inherited adapted-frame (\textbf{P2}), to convert a given spatial curve into a piecewise PH curve, while ensuring boundedness on the approximation error (\textbf{P1}) and continuity on the parametric functions (\textbf{P3}). In this section we provide further details on the underlying methodology and its compliance with these three features.

\subsection{Piecewise PH curves with $C^2$ parametric functions}\label{subsec:piecewise_ph_c2}
\noindent Paths parameterized by PHODCOS will only be applicable to second-order optimization methods if all parametric functions are, at least, twice differentiable i.e., $\{\bm{p},\text{R},\bm{\omega}\}\in C^2(\xi)$ (see \textbf{P3}). As mentioned earlier, we approximate the original curve with a piecewise PH curve, and thus, special care needs to be taken across intersections. More specifically, out of all three parametric functions, the angular velocity of the adapted-frame is the one with highest preimage degree, i.e., $\bm{\omega}(\xi) = f(\mcl{A}(\xi),\mcl{A}'(\xi))$. Therefore, $C^2$ in all parametric functions is equivalent to $C^3$ in the preimage $\mcl{A}(\xi)$, leading to a $C^4$ condition in the path-function $\bm{p}(\xi)$\footnote{Given that $\bm{\omega}(\xi)$ depends on $\mcl{A}'(\xi)$, it will only will be $C^2$ if, at least, $\mcl{A}'''(\xi)$ is continuous, implying that $\mcl{A}\in C^3(\xi)$. The hodograph $\bm{h}(\xi)$, and consequently also the preimage $\mcl{A}(\xi)$, relates by its derivative to the path-function $\bm{p}(\xi)$. Therefore, $\{\bm{h},\mcl{A}\}\in C^3(\xi)$ infers $\bm{p}\in C^4(\xi)$.}:
\begin{equation}\label{eq:continuity_condition}
    \{\sigma,\,\text{R},\,\bm{\omega}\} \in C^2(\xi) \Rightarrow \mcl{A}\in C^3(\xi)\Rightarrow \bm{p}\in C^4(\xi)
\end{equation}
\subsection{PH curves of degree 17}\label{subsec:PH_deg17}
\vspace{1mm}
\subsubsection{Defining the degree}

\noindent To satisfy the continuity condition~\eqref{eq:continuity_condition}, we split the original curve into multiple segments and approximate each of them by conducting a $C^4$ hermite interpolation. This is equivalent to specifying position ($0$), velocity ($1$), acceleration ($2$), jerk ($3$) and snap ($4$) starting at ending conditions, resulting in $30$\footnote{Amount of interpolation constraints per interval: $30=2(\text{start and end})\cdot3(\text{x, y and z})\cdot5(\text{position to snap})$} scalar constraints per interval. Imposing the starting position $\bm{p}_0$, related to the constant obtained from the integration of the hodograph~\eqref{eq:h_func} into~\eqref{eq:path_func}, reduces the amount of free conditions to $27$. Given the four-dimensional nature of the preimage $\mcl{A}(\xi)$, a preimage of order $7$ ($4\cdot7=28>27$) accounts for sufficient degrees of freedom. However, in a similar way to \cite{vsir2007}, increasing the order to $8$ guarantees the existence of a solution for any interpolation data. From~\eqref{eq:polynomial_funcs}, it can be stated that defining the order of preimage $\mcl{A}(\xi)$ as $n=8$ leads to a hodograph $\bm{h}(\xi)$ of degree $2n=16$ and a path-function $\bm{p}(\xi)$ of degree $2n+1=17$.

\subsubsection{Hodograph and path-function}

\noindent Combining the definition of the preimage in~\eqref{eq:A_func} with the quadratic expression in~\eqref{eq:hodograph} allows for associating the control-points of the hodograph $\bm{h}_{0,\dots,16}$ to the ones of the preimage $\mcl{A}_{0,\dots,8}$. After some simplifications\footnote{Intermediary steps for computing the hodograph's control-points can be found in Appendix~\ref{appendix:h_ctrlpts}.}, these are given by the following expressions:
\begin{subequations}\label{eq:h_ctrlpts}
\allowdisplaybreaks
\begin{align}
\bm{h}_0 &= \Asq{0}\,,\label{eq:h0}\\
\bm{h}_1 &= \AstA{0}{1}\,,\label{eq:h1}\\
\bm{h}_2 &= \frac{1}{15}(7\AstA{0}{2}+8\Asq{1})\,,\label{eq:h2}\\
\bm{h}_3 &= \frac{1}{10}(2\AstA{0}{3}+8\AstA{1}{2})\,,\label{eq:h3}\\
\bm{h}_4 &= \frac{1}{65}(5\AstA{0}{4}+32\AstA{1}{3}+28\Asq{2})\,,\label{eq:h4}\\
\bm{h}_5 &=\frac{1}{39}(\AstA{0}{5}+10\AstA{1}{4}+28\AstA{2}{3})\label{eq:h5}\,,\\
\begin{split}
\bm{h}_6 &= \frac{1}{143}(\AstA{0}{6}+16\AstA{1}{5}+70\AstA{2}{4}+\\
&\quad56\Asq{3})\label{eq:h6}\,,
\end{split}\\
\begin{split}
\bm{h}_7 &= \frac{1}{715}(\AstA{0}{7}+28\AstA{1}{6}+196\AstA{2}{5}+\\
&\quad490\AstA{3}{4})\label{eq:h7}\,,
\end{split}\\
\begin{split}
\bm{h}_8 &= \frac{1}{6435}(\AstA{0}{8} + 64\AstA{1}{7} + 784\AstA{2}{6} +\\
 &\quad3136\AstA{3}{5} + 2450\Asq{4})\,,
\end{split}\\
\begin{split}
\bm{h}_9 &= \frac{1}{715}(\AstA{1}{8}+28\AstA{2}{7}+196\AstA{3}{6}+\\
&\quad490\AstA{4}{5})\label{eq:h9}\,,
\end{split}\\
\begin{split}
\bm{h}_{10} &= \frac{1}{143}(\AstA{2}{8}+16\AstA{3}{7}+70\AstA{4}{6}+\\
&\quad56\Asq{5})\label{eq:h10}\,,
\end{split}\\
\bm{h}_{11} &= \frac{1}{39}(\AstA{3}{8}+10\AstA{4}{7}+28\AstA{5}{6})\label{eq:h11}\,,\\
\bm{h}_{12} &= \frac{1}{65}(5\AstA{4}{8}+32\AstA{5}{7}+28\Asq{6})\,,\label{eq:h12}\\
\bm{h}_{13} &= \frac{1}{10}(2\AstA{5}{8}+8\AstA{6}{7})\,,\label{eq:h13}\\
\bm{h}_{14} &= \frac{1}{15}(7\AstA{6}{8}+8\Asq{7})\,,\label{eq:h14}\\
\bm{h}_{15} &= \AstA{7}{8}\,,\label{eq:h15}\\
\bm{h}_{16} &= \Asq{8}\,.\label{eq:h16}
\end{align}
\end{subequations}
Putting together these control-points with eq.~\eqref{eq:p_ctrlpts} and a predetermined starting point $\bm{p}_0$, the path-function's control-points $\bm{p}_{1,\dots,17}$ -- and as a result, the path-function itself -- are fully defined.

\subsection{Construction of $C^4$ hermite interpolants}\label{subsec:construction_c4}
\noindent In this subsection we provide a thorough description on how each interval of the original curve is approximated by a piecewise PH curve. 
\subsubsection{Boundary conditions}
\noindent At every segment, given $C^4$ hermite boundary data -- end points $\{\bm{p}_b,\bm{p}_e\}$\footnote{For a given parametric function $f(\xi)$, $f_b=f(0)$ and $f_e=f(1)$.}, and the first four derivative vectors  $\{\bm{v}_b,\bm{v}_e,\bm{a}_b,\bm{a}_e,\bm{j}_b,\bm{j}_e, \bm{s}_b,\bm{s}_e\}$ --, we construct a spatial PH curve $\bm{p}(\xi)$.
Considering the  Bernstein-Bezier form of the hodograph $\bm{h}(\xi)$ in~\eqref{eq:h_func} and its relationship with respect to the interpolant data, i.e., $\bm{v}(\xi) = \bm{h}(\xi),\,  \bm{a}(\xi) = \bm{h}'(\xi),\, \bm{j}(\xi) = \bm{h}''(\xi),\, \bm{s}(\xi) = \bm{h}'''(\xi)$, it results that
\begin{subequations}\label{eq:c4_data}
\begin{align}
\bm{v}_b &= \bm{h}_0\,,\label{eq:vb}\\
\bm{v}_e &= \bm{h}_{16}\,,\label{eq:ve}\\
\bm{a}_b &= -16\bm{h}_0+16\bm{h}_1\,,\label{eq:ab}\\
\bm{a}_e &= -16\bm{h}_{15}+16\bm{h}_{16}\,,\label{eq:ae}\\
\bm{j}_b &= 240\bm{h}_0-480\bm{h}_1+240\bm{h}_2\,,\label{eq:jb}\\
\bm{j}_e &= 240\bm{h}_{14}-480\bm{h}_{15}+240\bm{h}_{16}\,,\label{eq:je}\\
\bm{s}_b &= -3360\bm{h}_0 + 10080 \bm{h}_1 - 10080 \bm{h}_2 + 3360 \bm{h}_3\,,\label{eq:sb}\\
\begin{split}
\bm{s}_e &= -3360\bm{h}_{13} + 10080 \bm{h}_{14} - 10080 \bm{h}_{15}+\\
&\quad3360 \bm{h}_{16}\,.\label{eq:se}
\end{split}
\end{align}
\end{subequations}
Moreover, introducing the interpolation endpoints on~\eqref{eq:p_ctrlpts} leads to
\begin{equation*}
    \bm{p}_e-\bm{p}_b=\frac{1}{17}\sum_{i=0}^{16}\bm{h}_i\,.
\end{equation*}
and combining it with~\eqref{eq:h_ctrlpts}, after some simplifications, results in
\begin{subequations}\label{eq:A4_formula}
\begin{equation}
    \mcl{A}_{\bm{p}}^{2\star} = \frac{490}{21879}\left(\bm{p}_e-\bm{p}_b\right)-\bm{c}_{\bm{p}}
\end{equation}
with
\begin{align}
\begin{split}
    \mcl{A}_{\bm{p}} &= \frac{1}{442}\mcl{A}_{0}+\frac{5}{663}\mcl{A}_{1}+\frac{35}{2431}\mcl{A}_{2}+\frac{49}{2431}\mcl{A}_3+\\
    &\quad\frac{490}{21879}\mcl{A}_4+\frac{49}{2431}\mcl{A}_5+\frac{35}{2431}\mcl{A}_6+\frac{5}{663}\mcl{A}_7+\\
    &\quad\frac{1}{442}\mcl{A}_8
\end{split}
\end{align}
and
\begin{align}\label{eq:cp_star}
\begin{split}
\bm{c}_{\bm{p}} &= 
\frac{1}{112633092}(147807 \Asq{0}+ 144540\AstA{0}{1} +\\
&\quad61908 \AstA{0}{2}+19404 \AstA{0}{3}-6468 \AstA{0}{5}-\\
&\quad6300 \AstA{0}{6}-3636 \AstA{0}{7}-1130 \AstA{0}{8}+\\
&\quad+72732 \Asq{1}+94248 \AstA{1}{2}+38808 \AstA{1}{3}-\\
&\quad17640 \AstA{1}{5}-18648 \AstA{1}{6}-11336 \AstA{1}{7}-\\
&\quad3636 \AstA{1}{8}+40572 \Asq{2}+41160 \AstA{2}{3}-\\
&\quad24696 \AstA{2}{5}-28616 \AstA{2}{6}-18648 \AstA{2}{7}-\\
&\quad6300 \AstA{2}{8}+12348 \Asq{3}-19208 \AstA{3}{5}-\\
&\quad24696 \AstA{3}{6}-17640 \AstA{3}{7}-6468 \AstA{3}{8}+\\
&\quad12348 \Asq{5}+41160 \AstA{5}{6}+d38808 \AstA{5}{7}+\\
&\quad19404 \AstA{5}{8}+40572 \Asq{6}+94248 \AstA{6}{7}+\\
&\quad61908 \AstA{6}{8}+72732 \Asq{7}+144540 \AstA{7}{8}+\\
&\quad147807 \Asq{8})\,.
\end{split}
\end{align}
\end{subequations}

\subsubsection{Standard form}
\noindent As discussed later, to ensure that the interpolation method is fully invariant, the interpolation data needs to be transformed into a standard form.
\begin{definition}[\bfseries Standard form]\label{definition:standard_form}
The $C^4$ spatial hermite data are considered to be in standard form if $\bm{v}_b + \bm{v}_e$ can be expressed as a positive multiple of $\bi$ and $\bm{p}_b=\bm{0}$.
\end{definition}

\subsubsection{The interpolation algorithm}
\noindent Having defined the preliminaries, boundary conditions and data's standard form, the steps to compute the preimage associated with a PH curve that interpolates $C^4$ hermite data are described in Algorithm~\ref{algo:c4}.

\begin{algorithm}[h]
\caption{$C^4$ \textit{hermite data interpolation}:\newline Given interpolation data $\{\bm{p}_b,\bm{p}_e, \bm{v}_b, \bm{v}_e, \bm{a}_b, \bm{a}_e, \bm{j}_b, \bm{j}_e, \bm{s}_b, \bm{s}_e\}$, compute preimage control-points $\{\mcl{A}_{0},\allowbreak\mcl{A}_1, \allowbreak\mcl{A}_2, \allowbreak\mcl{A}_3, \allowbreak\mcl{A}_4, \allowbreak\mcl{A}_5, \allowbreak\mcl{A}_6, \allowbreak\mcl{A}_7, \allowbreak\mcl{A}_8\}$ .}\label{algo:c4}
\begin{algorithmic}[1]
\Function{C4Interpolation}{$\bm{p}_b,\bm{p}_e, \bm{v}_b, \bm{v}_e, \bm{a}_b, \bm{a}_e, \bm{j}_b, \bm{j}_e, \allowbreak\bm{s}_b, \bm{s}_e$}
\State  $\{\bm{p}_{b},\cdots,\bm{s}_{e}\}_{sf}\gets$Transform the data to the standard form in Definition~\ref{definition:standard_form} by a rotation and translation.
\State $\bm{h}_0,\bm{h}_1,\bm{h}_2,\bm{h}_3,\bm{h}_{13},\bm{h}_{14},\bm{h}_{15},\bm{h}_{16}\gets$ Compute hodograph control-points from~\eqref{eq:c4_data}
\State $\mcl{A}_0,\mcl{A}_8 \gets$ Calculate $\mcl{A}_0$ and $\mcl{A}_8$ from eqs.~\eqref{eq:h0} and~\eqref{eq:h16}, which follow Lemma~\ref{lemma:A2}, and thus, introduce free parameters $\theta_0,\theta_8$.
\State $\mcl{A}_1,\mcl{A}_7 \gets$ Calculate $\mcl{A}_1$ and $\mcl{A}_7$ from eqs.~\eqref{eq:h1} and~\eqref{eq:h15}, which follow Lemma~\ref{lemma:AB}, and thus, introduce free parameters $\tau_1,\tau_7$.
\State $\mcl{A}_2,\mcl{A}_6 \gets$ Calculate $\mcl{A}_2$ and $\mcl{A}_6$ from eqs.~\eqref{eq:h2} and~\eqref{eq:h14}, which follow Lemma~\ref{lemma:AB}, and thus, introduce free parameters $\tau_2,\tau_6$.
\State $\mcl{A}_3,\mcl{A}_5 \gets$ Calculate $\mcl{A}_3$ and $\mcl{A}_5$ from eqs.~\eqref{eq:h3} and~\eqref{eq:h13}, which follow Lemma~\ref{lemma:AB}, and thus, introduce free parameters $\tau_3,\tau_5$.
\State $\mcl{A}_4 \gets$Calculate $\mcl{A}_4$ from eq.~\eqref{eq:A4_formula}, which follow Lemma~\ref{lemma:A2}, and thus, introduces a free parameter $\theta_4$.
\State $\bm{h}_{4,\dots,12} \gets$ From eqs.~\eqref{eq:h4}-~\eqref{eq:h12}, compute $\bm{h}_{4,\dots,12}$ and, by setting $\bm{p_b}=\bm{p_0}$ calculate the remaining control-points of $\bm{p}(\xi)$.
\State $\mcl{A}_{0,\cdots,8} \gets$ Return the solution to its original pose, by performing the inverse operation of step 1.  
\State\Return $\mcl{A}_{0},\allowbreak\mcl{A}_1, \allowbreak\mcl{A}_2, \allowbreak\mcl{A}_3, \allowbreak\mcl{A}_4, \allowbreak\mcl{A}_5, \allowbreak\mcl{A}_6, \allowbreak\mcl{A}_7, \allowbreak\mcl{A}_8$
\EndFunction
\end{algorithmic}
\end{algorithm}
\begin{remark}
To be more specific, in line 8 the expression to compute $\mcl{A}_4$ is
\begin{align}
        \begin{split}
        \mcl{A}_4 &= \frac{21879}{490}\left(\sqrt{|\bm{\bm{c}_{\bm{p}}}|}\frac{\frac{\bm{\bm{c}_{\bm{p}}}}{|\bm{\bm{c}_{\bm{p}}}|}+\bm{i}}{|\frac{\bm{\bm{c}_{\bm{p}}}}{|\bm{\bm{c}_{\bm{p}}}|}+\bm{i}|}\mcl{Q}(\theta_4)-(\frac{1}{442}\mcl{A}_{0}+\right.\\
        &\left.\quad\frac{5}{663}\mcl{A}_{1}+\frac{35}{2431}\mcl{A}_{2}+\frac{49}{2431}\mcl{A}_3+\frac{49}{2431}\mcl{A}_5+\right.\\
        &\left.\quad\frac{35}{2431}\mcl{A}_6+\frac{5}{663}\mcl{A}_7+\frac{1}{442}\mcl{A}_8)\right) \,.\\
        \end{split}
\end{align}
\end{remark}
\noindent As a consequence, the preimages $\mcl{A}_\phi(\xi)$, hodographs $\bm{h}_\phi(\xi)$ and path-functions $\bm{p}_\phi(\xi)$ of a PH curve that interpolate the given $C^4$ hermite data relate to a $9$ parameter family $\bm{\phi} = \left[\theta_0,\tau_1,\tau_2,\tau_3,\theta_4,\tau_5,\tau_6,\tau_7,\theta_8\right]$. Similar to C1 and C2 hermite interpolation~\cite{vsir2005spatial, vsir2007}, one of the angular parameters, e.g. $\theta_4$, can be set to zero owing to the non-trivial fibers of the map from the preimage to the hodograph~\eqref{eq:hodograph}.As a result, all PH interpolants can be derived using parameter vectors where $\theta_4 = 0$. Considering this, for the remainder of the paper, we will adopt this specific option and remove $\theta_4$ from the parameter vector $\phi$, leading to $\phi = [\theta_0, \tau_1, \tau_2, \tau_3, \tau_5, \tau_6, \tau_7, \theta_8]$.

\subsection{Choosing the optimal interpolants} 
\noindent In this section we conduct a qualitative analysis of the system of PH curve interpolants $\bm{\phi}$ in order to identify the “best” values of the parameter vector $\bm{\phi}^*$. To fix the free parameters $\bm{\phi}$, we examine the asymptotic behavior of the solutions $\bm{p}_\phi(\xi)$. Specifically, we consider the $C^4$ hermite data extracted from a short portion of an analytic curve and study the asymptotic properties of the solutions as the step size diminishes.

Considering that the curve is represented by its Taylor series, without any loss of generality,
\begin{equation}
\bm{\gamma}(\Xi)= \left(\Xi+\sum_{i=2}^{\infty}\frac{x_i}{i!}\Xi^i,\sum_{i=2}^{\infty}\frac{y_i}{i!}\Xi^i,\sum_{i=2}^{\infty}\frac{z_i}{i!}\Xi^i\right)^\trsp\,,
\end{equation}
where $x_{i},\,y_{i},\,z_{i}$ with $i=2,\dots,\infty$ refer to arbitrary coefficients. For a given step size of $h$, we pick the segment $k$, $\bm{\gamma}(\Xi)=\bm{\gamma}_k(h\xi),\,\xi\in[0,1]$, whose expansion is
\begin{equation}\label{eq:segmented_curve}
\bm{\gamma}_k(\xi) =  \left(\xi h+\sum_{i=2}^{\infty}\frac{x_i}{i!}\xi^i h^i,\sum_{i=2}^{\infty}\frac{y_i}{i!}\xi^i h^i,\sum_{i=2}^{\infty}\frac{z_i}{i!}\xi^i h^i\right)^\trsp\,.
\end{equation}
Now we interpolate the $C^4$ hermite boundary data at the points $\bm{\gamma}_k(0) = \bm{\gamma}(0)$ and $\bm{\gamma}_k(1) = \bm{\gamma}(h)$. The behavior of different PH curves that interpolate the data, based on the size of the interval $h$, is outlined in the theorem below.
\begin{theorem}[\bfseries Approximation error]\label{theorem:approximation_error}
The interpolation error of the PH
\begin{equation}
E = \max_{\xi\in[0,1]} ||\bm{\gamma}_k(\xi)-\bm{p}_\phi (\xi)||
\end{equation}
converges to 0 for $h\rightarrow 0$ if and only if $\tau_1 = \tau_2 = \tau_3 = \tau_5 = \tau_6 = \tau_7 = 0$. Moreover, if $\theta_0 = \theta_4 = \theta_8 = 0$, then $E = \mcl{O}\left(h^{6}\right)$. Otherwise, $E = \mcl{O}\left(h^{1}\right)$.
\end{theorem}
\begin{proof}
The proof entails the assessment of the power series for all quantities involved in the interpolation procedure in relation to the step-size $h$, which can be achieved using a suitable computer algebra tool. Due to constraints in available space and the intricate nature of the expressions, we solely present the principal terms of specific quantities, aiming to demonstrate the underlying concept of our approach. For a detailed analysis of the proof, please refer to the Mathematica notebook attached alongside the released code \footnote{\label{fn:theory}Mathematica notebook with the proof of theorem~\ref{theorem:approximation_error}: \url{https://github.com/jonarriza96/phodcos/tree/main/theory}}. 
\end{proof}
\noindent Beyond ensuring the highest approximation order, choosing $\bm{\phi^*}=\bm{0}$ offers other noteworthy advantages, such as maintaining planarity, preserving interpolants, and being reversion invariant, as elaborated in Theorems \ref{theorem:planarity}, \ref{theorem:invariance}, and \ref{theorem:reversed}. These theorems and their corresponding proofs, which can be directly traced back to the $C^2$ predecessor~\cite{vsir2007}, are formally presented in Appendix~\ref{appendix:collorary_theorems}. When combined with Theorem~\ref{theorem:approximation_error}, it becomes clear that the interpolant $\bm{\phi}^*=\bm{0}$ is exceptionally effective for curve approximation. All these attributes collectively are encapsulated in the following statement.
\begin{corollary}[\bfseries Optimal interpolant $\bm{\phi^*}=\bm{0}$]\label{col:c4}
The $C^4$ hermite PH interpolant $\bm{p}_{\phi^*}(\xi)$ has an approximation order of $6$ and, in addition, retains the planarity of the given data, remains invariant under any orthogonal transformations, scaling, and reversion of the input.
\end{corollary}
\subsection{Continuity of the preimage and the adapted-frame}\label{subsec:continuity_preimage}
\noindent In the previous subsection, we have shown the benefits of choosing $\bm{\phi}^*=\bm{0}$ for the $C^4$ interpolation algorithm in ~\ref{algo:c4}. Before doing so, as explained at the end of subsection~\ref{subsec:construction_c4}, due to non-trivial fibers between the preimage and the hodograph, we have fixed the parameter $\theta_4=0$. Despite being a numerically convenient choice and not influencing the algorithm's approximation error, it implies that the preimage $\mcl{A}(\xi)$, as well as the adapted-frame $\text{R}(\xi)$, of two successive segments differ by a \emph{roll-angle offset}, and thus, 
\begin{subequations}\label{eq:frame_cont}
\begin{align*}
    \mcl{A}_k(1) \neq \mcl{A}_{k+1}(0)\quad\text{and}\quad\text{R}_k(1) \neq \text{R}_{k+1}(0)\,,
\end{align*}
where $k$ and $k+1$ refer to the successive segments. To recover the continuity, we perform a \emph{roll-rotation} $\alpha$ that removes the existing offset. This can be achieved by a two stage procedure: First, we calculate a quaternion that cancels the offset, 
    \begin{equation}
        \mcl{A}_{\alpha} = \delta_\alpha(\text{R}_k,\text{R}_{k+1})\,,
    \end{equation}
where $\delta_\alpha\,:\,\{\sothree,\sothree\}\rightarrow \sothree $ computes the $\alpha$ roll difference between two rotation matrices and outputs the respective quaternion. Second, we rotate all the preimage control-points accordingly
\begin{equation}\label{eq:rotated_preimage_ctrlpts}
    \mcl{A}_{0,\cdots,8} = \mcl{A}_{\alpha}\,\mcl{A}_{0,\cdots,8}\,.
\end{equation}
Reconstructing the PH curve with the rotated preimage control-points in~\eqref{eq:rotated_preimage_ctrlpts} results in a parametric path that, besides having the appealing mathematical attributes summarized in Corollary~\ref{col:c4}, it is also offset-free, and thereby continuous, in its preimage $\mcl{A}(\xi)$ and adapted-frame $\text{R}(\xi)$.
\end{subequations}


\begin{figure*}[!ht]
\centering
\includegraphics[width=0.91\textwidth]{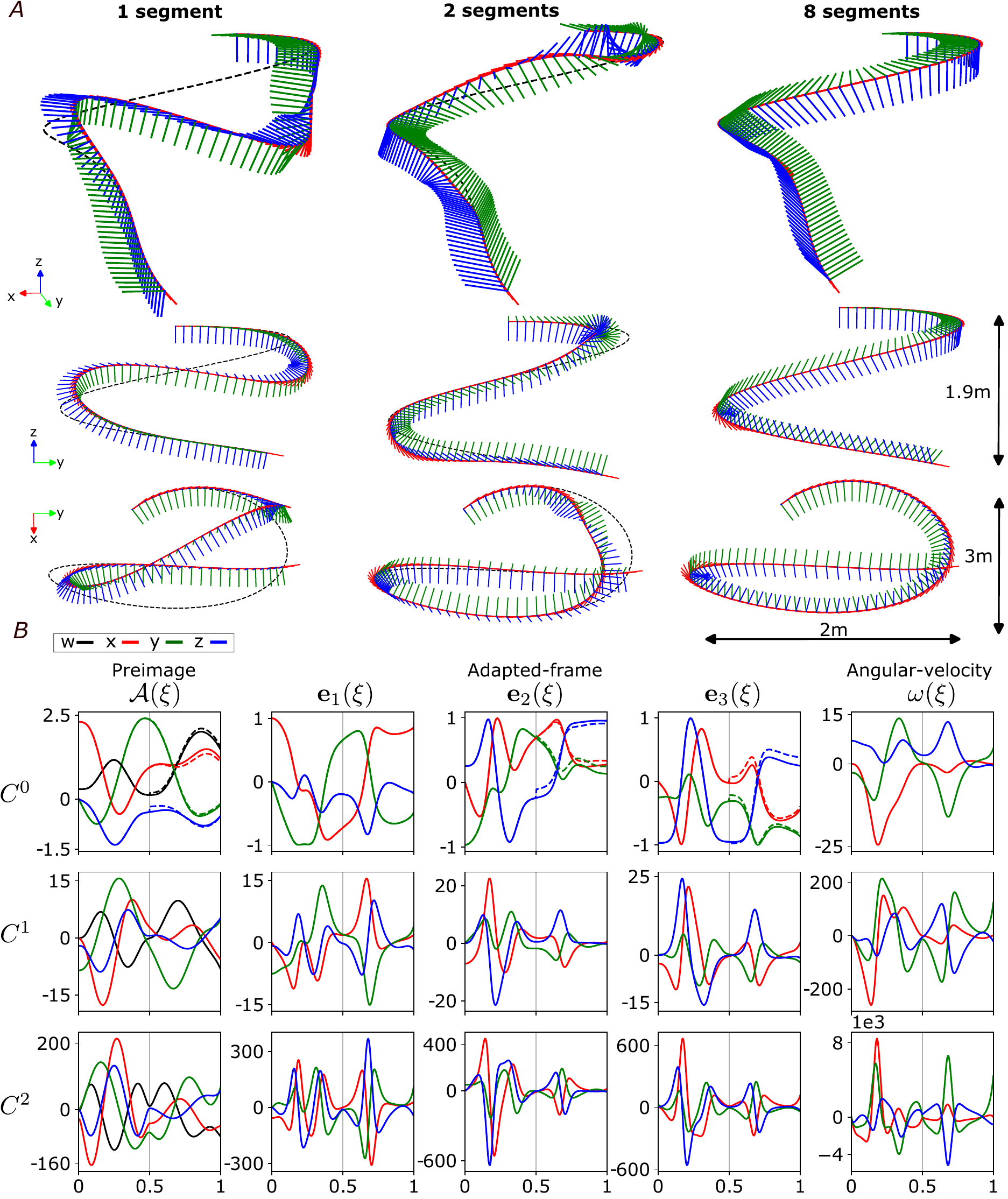}
\caption{\textit{A)} Comparison of the parametric paths resulting from the conversion of an exemplary analytical curve~\eqref{eq:exemplary_curve} by means of PHODCOS -- the algorithm~\ref{algo:papor} presented in this work --  for 1, 2 and 8 segments. The original analytical curve is given by the dashed black line, while the converted parametric path is depicted by its adapted frame, with the first, second and third components shown in red, green and blue, respectively. The computed parametric paths are divided into columns, each of which displays the results associated with a specific number of segments. In each column, the top row provides an isometric view, while the middle and bottom rows offer side and top views, respectively. The conversion error and improvement ratios associated to every number of segments are given in Table~\ref{tab:benchmark}. As stated in theorem~\ref{theorem:approximation_error}, the conversion error decreases as the number of segments augments, resulting in a greater similarity between the parametric path and the original curve. \textit{B)}. Continuity analysis in the preimage $\mcl{A}(\xi)$ (first column), adapted frame $\text{R}(\xi)$ (second, third and fourth columns) and angular-velocity $\bm{\omega}(\xi)$ (fifth column) for the \emph{2 segments} case. The transition between the first and second segments is given by the gray vertical line at $\xi=0.5$. The first row depicts the functions themselves, while the second and third relate to the first and second derivatives. The dashed lines intend to expose the loss on continuity in the parametric functions if the angle offset cancellation~\eqref{eq:frame_cont} is not conducted.}\label{fig:benchmark}
\vspace{-5mm}
\end{figure*}
\subsection{Overview of the method}
\begin{algorithm}[h]
\small
\caption{\textit{Pythagorean Hodograph-based Differentiable Coordinate System} (\textbf{PHODCOS}):\newline Given an analytical curve $\bm{\gamma}(\xi)$ and an admissible error tolerance $\epsilon$, convert it into a parametric-path whose Cartesian coordinates $\bm{p}(\xi)$, adapted-frame $\text{R}(\xi)$, angular velocity $\bm{\omega}(\xi)$, are given by -- at least -- $C^2$ functions.}\label{algo:papor}
\begin{algorithmic}[1]
\Function{PHODCOS}{$\bm{\gamma}(\xi)$,\,$\epsilon$}
\State $e = \infty\,,\,n_s = 2$
\While{$e\geq\epsilon$}
\State $\bm{\gamma}_0,...\,,\bm{\gamma}_{n_s} \gets \textsc{DivideIntoSegments}(\bm{\gamma},n_s)$
\For{$k\in\{1,\,...,n_s\}$}
\State $\bm{p}_b,\cdots,\bm{s}_e \gets \textsc{GetInterpolationData}(\bm{\gamma}_k)$
\State $\mcl{A}^k_0,\cdots,\mcl{A}^k_8 \gets \textsc{C4hermite}(\bm{p}_b,\cdots,\bm{s}_e)$~\ref{algo:c4}
\State $\mcl{A}^k_0,\cdots,\mcl{A}^k_8 \gets \textsc{ERFConti}(\mcl{A}^k_0,\cdots,\mcl{A}^k_8)$~\eqref{eq:frame_cont}
\EndFor
\State $\bm{p},\text{R}, \bm{\omega}, \cdots \gets \textsc{ConstructCurve}(\mcl{A})$\eqref{eq:polynomial_funcs}\eqref{eq:erf}\eqref{eq:angvel}
\State $e \gets \textsc{CalculateConversionError}(\bm{\gamma},\bm{p})$
\State $n_s \gets n_s+1$
\EndWhile
\State\Return $\mcl{A}_{0,\cdots,8},\,\bm{p}(\xi),\text{R}(\xi), \bm{\omega}(\xi),L(\xi),\kappa(\xi),\tau(\xi),\cdots$
\EndFunction
\end{algorithmic}
\end{algorithm}
\noindent The pseudocode for PHODCOS -- the presented path parameterization scheme -- is summarized in Algorithm~\ref{algo:papor}. Provided an analytical curve $\bm{\gamma}(\xi)$ and a permissible error tolerance $\epsilon$, PHODCOS converts it into a parametric path, whose Cartesian coordinates $\bm{p}(\xi)$, adapted-frame $\text{R}(\xi)$, angular velocity $\bm{\omega}(\xi)$ are described by -- at least -- $C^2$ functions.  Additionally, parametric functions expressing other geometric characteristics such as arc length $L(\xi)$, curvature $\kappa(\xi)$, and torsion $\tau(\xi)$ can also be computed in closed-form. Due to the approximation error of order $h^6$, proven in theorem~\ref{theorem:approximation_error}, the algorithm iteratively runs the $C^4$ hermite interpolation method in~\ref{algo:c4} by augmenting the amount of segments $n_s$ until the approximation error $e$ is below the desired tolerance $\epsilon$. Notice that the presented iterative scheme to find the minimum amount of segments $n_s$ that fulfills the tolerance $\epsilon$ is exemplary and could be further tailored to the user needs, either by starting with a better initial guess -- higher $n_s$ -- or by modifying the update strategy at every iteration.
\begin{figure*}[!h]
\centering
\includegraphics[width=\textwidth]{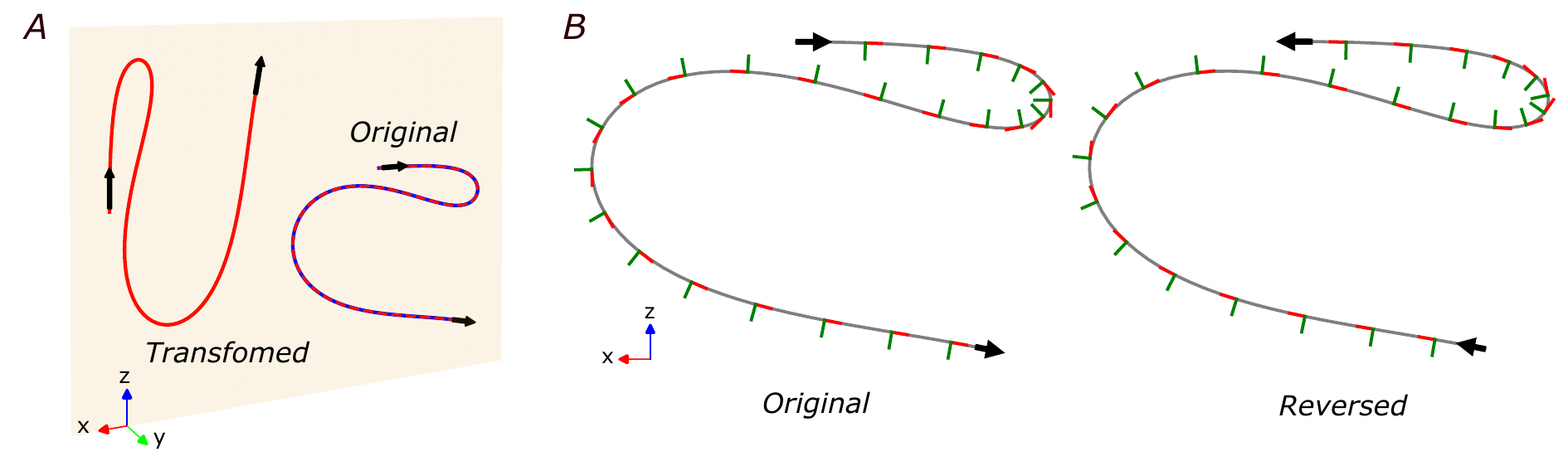}
\caption{Numerical validation for the planarity, invariance and reversion theorems~\ref{theorem:planarity}, ~\ref{theorem:invariance} and~\ref{theorem:reversed}. The original curve is a planar variant of~\eqref{eq:exemplary_curve}, whose $y$ component is set to $0$. In both figures the interpolation data is represented by black arrows. \textit{A)} Within the $xz$ plane colored in yellow, the \emph{original curve} is given in blue, the \emph{transformed curve} in red and the curve resulting from rotating and translating back the transformed curve in dashed-red. This curve aligns with the original one, and thereby, numerically validates theorem~\ref{theorem:invariance}. Additionally, in accordance with the planarity theorem~\ref{theorem:planarity}, all curves remain in the same plane as the interpolation data. \textit{B)} Planar view of the original and reversed curves, alongside their adapted frames, where the first and second components are colored in red and green, respectively. As stated by the reversion theorem~\ref{theorem:reversed}, the curves obtained for the optimal interpolants $\bm{\phi}^*=0$ are equivalent and can only be distinguished by the direction of their adapted-frames.}\label{fig:theorems}
\end{figure*}
\subsection{Numerical analysis}
\noindent The method described in the preceding subsection enables the transformation of any analytical curve into a PH curve. To showcase the properties of this algorithm, we perform a numerical analysis consisting of three parts. Firstly, we concentrate on an exemplary curve to replicate the approximation error as proven in theorem~\ref{theorem:approximation_error}. Secondly, we examine the continuity of the resulting parametric functions. Lastly, we investigate the planarity, reversion, and invariance properties outlined in Theorems~\ref{theorem:planarity}, \ref{theorem:invariance}, and \ref{theorem:reversed}.

\vspace{1mm}
\subsubsection*{Approximation error}
\hspace{1cm}

\vspace{1mm}
\noindent To numerically replicate the approximation order obtained in theorem~\ref{theorem:approximation_error}, as an illustrative example we select the same curve as in~\cite{vsir2005spatial,vsir2007}:
\begin{equation}\label{eq:exemplary_curve}
    \bm{\lambda}(\xi) = \left[1.5\sin(7.2\xi), \cos(9\xi), e^{\cos(1.8\xi)}\right]\,,
\end{equation}
for $\xi\in \left[0,1\right]$. We split this interval into $2^n$ uniform segments, i.e., $h=\frac{1}{2^n}$. By means of the $C^4$ interpolation algorithm~\ref{algo:c4}, each of them is converted into a PH curve. When doing so, aiming for the same approximation order as in theorem~\ref{theorem:approximation_error}, we select $\bm{\phi}^*=\bm{0}$. If the error between the original analytical curve and the converted PH curve is not sufficiently small, we continue with further subdivision by increasing $n$. This implies that the error is expected to converge to $0$ as $\mcl{O}(h^6) = \mcl{O}(\frac{1}{64^n})$.  Please note that achieving the desired convergence in the improvement ratio necessitates performing computations with high precision. In order to accomplish this, we utilize the advanced precision capabilities offered by Mathematica and an accompanying notebook containing the relevant calculations is included along with the released code\footref{fn:theory}.

The maximum distance between the original curve and the converted path -- denoted as \emph{conversion error} -- for different divisions, alongside the respective improvement ratios are given in Table~\ref{tab:benchmark}. As shown in the third and sixth columns, the improvement ratio converges to $64$. Besides that, in Fig.~\ref{fig:benchmark}-A, a comparison of the parametric paths resulting from the conversion with $1$, $2$ and $8$ segments, intuitively depicts how a higher number of segments correlates to a better approximation error.

\begin{table}[t]
    \setlength{\tabcolsep}{6pt} 
    \renewcommand{\arraystretch}{1.25} 
    \centering
    \caption{Maximum distance error and improvement ratio between the original curve and the parametric path obtained from PHODCOS. The error decreases by a ratio of 64, as stated in theorem~\ref{theorem:approximation_error}. Paths for 1, 2, and 8 segments are shown in Fig.~\ref{fig:benchmark}.} \label{tab:benchmark}
    \begin{tabular}{|c|c|c|}
        \hline
        \# Segments & Error [m] & Ratio \\
        \hline
        1 & 1.2569 & - \\
        2 & 0.5447 & 2.307 \\
        4 & 0.0332 & 16.405 \\
        8 & $\expnumber{16.080}{-4}$ & 20.646 \\
        16 & $\expnumber{24.455}{-6}$ & 63.173 \\
        32 & $\expnumber{1.897}{-7}$ & 134.145 \\
        64 & $\expnumber{5.009}{-9}$ & 37.878 \\
        128 & $\expnumber{8.009}{-11}$ & 62.545 \\
        256 & $\expnumber{1.272}{-12}$ & 62.986 \\
        512 & $\expnumber{1.982}{-14}$ & 64.161 \\
        1024 & $\expnumber{3.105}{-16}$ & 63.822 \\
        2048 & $\expnumber{4.856}{-18}$ & 63.955 \\
        4096 & $\expnumber{7.588}{-20}$ & 63.989 \\
        8192 & $\expnumber{1.186}{-23}$ & 63.997 \\
        \hline
    \end{tabular}
\end{table}

\vspace{1mm}
\subsubsection*{Continuity analysis}
\hspace{1cm}

\vspace{1mm}
\begin{figure*}[!h]
\centering
\includegraphics[width=\textwidth]{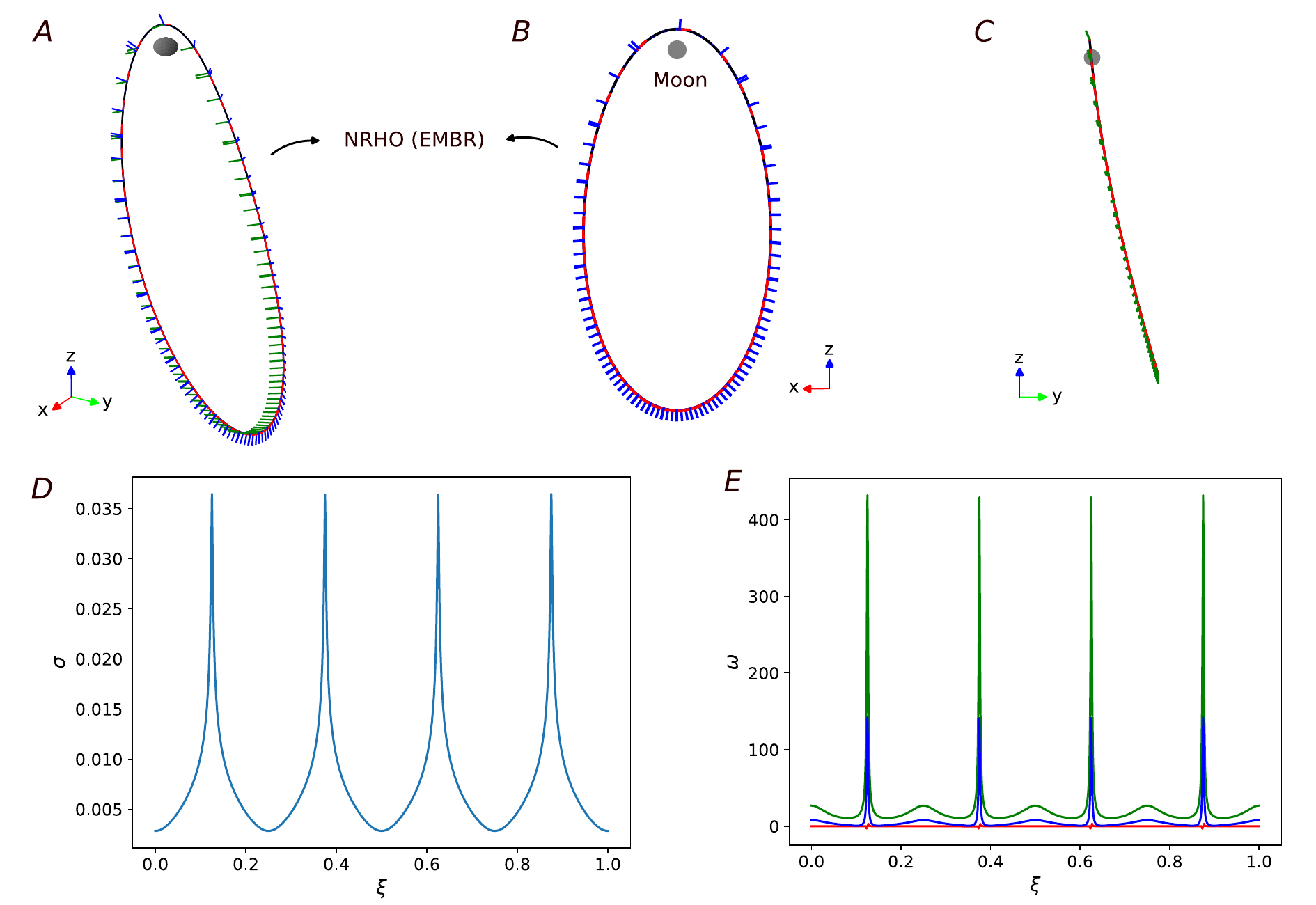}
\caption{The PHODCOS coordinate system is applied to the Near Rectilinear Halo Orbit (NRHO) within the Earth-Moon Barycentric Rotating (EMBR) reference frame. This system is illustrated from three perspectives: A) isometric view, B) front view, and C) side view. The vector components corresponding to the coordinate system are represented in red, blue, and green, respectively. The parametric speed $\sigma$ and the angular velocity $\bm{\omega}$ associated with the coordinate system are depicted in panels D) and E). For an alternative view of the coordinate system from the Earth-Centered Inertial (ECI) frame, please refer to Fig.~\ref{fig:nrho_cover}.}\label{fig:nrho}
\end{figure*}

\noindent In order to showcase that the parametric functions resulting from PHODCOS achieve the desired continuity of $C^2$, we analyze their continuity up to the second derivative. For this purpose we focus on the case with two segments, namely the middle column in Fig.~\ref{fig:benchmark}-A. Despite its high approximation error, this case allows us to clearly observe the transition between the first and second segments. 

The preimage $\mcl{A}(\xi)$ and the parametric functions for the adapted frame $\text{R}(\xi)$ and its angular velocity $\bm{\omega}(\xi)$, alongside their first two derivatives, are given in Fig.~\ref{fig:benchmark}-B. Notice that we do not analyze $\bm{p}(\xi)$ since this is apriori known to be $C^4$ from the underlying interpolation algorithm~\ref{algo:c4}. The transition between the first and second segment occurs at $\xi=0.5$ and is depicted by a grey line. As these plots suggest, all parametric functions are at least $C^2$. More specifically, following the theoretical explanations in~\ref{subsec:piecewise_ph_c2},  the angular velocity $\bm{\omega}(\xi)$ is the parametric function with the lowest degree of continuity of $C^2$. This can be observed from the plot at the bottom right side of Fig.~\ref{fig:benchmark}-B, where further derivation would result in discontinuities at the transition $\xi=0.5$. This is not the case for the preimage $\mcl{A}(\xi)$ and the adapted frame $\text{R}(\xi)$, which are $C^3$, and therefore, still account for an additional continuous derivative.

Besides that, we also study the consequences of omitting the rotation -- explained in subsection~\ref{subsec:continuity_preimage} -- that ensures continuity in the preimage $\mcl{A}(\xi)$ and the adapted frame $\text{R}(\xi)$ by removing the roll-offset between two successive segments. The resultant parametric functions are depicted by dashed lines in the first row of Fig.~\ref{fig:benchmark}-B and, for clarity, we do not show the higher derivatives. From the obtained results, it becomes apparent that without the angle-offset cancellation, despite having a $C^4$ parametric path function $\bm{p}(\xi)$, continuity in the preimage $\mcl{A}(\xi)$ and the adapted frame $\text{R}(\xi)$ is lost, thereby rendering the offset cancellation in subsection~\ref{subsec:continuity_preimage} indispensable.

\subsubsection*{Planarity, reversion and invariance analysis}

\vspace{1mm}
\noindent In this subsection we numerically validate the remaining theorems~\ref{theorem:planarity}, ~\ref{theorem:invariance}, and~\ref{theorem:reversed}. For this purpose, we reuse the exemplary curve~\eqref{eq:exemplary_curve} from the previous subsections. However, to validate the planarity preservation in theorem~\ref{theorem:planarity}, we set the $y$ component to $0$. Using the planar projection, we separately study the invariance to special orthogonal motions in theorem~\ref{theorem:invariance} and reversion in theorem~\ref{theorem:reversed}. In both cases, we consider the original curve~\eqref{eq:exemplary_curve} for a single segment. 

The first study relates to Fig.~\ref{fig:theorems}-A, where we manipulate the hermite data of the \emph{original curve} (depicted in blue) by employing an arbitrary special orthogonal motion, specifically a translation of $[-1,0,2.5]$ and a rotation of $\pi/2$ radians around the $y$ axis. The interpolation of this data yields a \emph{transformed curve} (given in red). According to theorem~\ref{theorem:invariance}, it is affirmed that by translating and rotating the transformed curve back, an equivalent curve to the original one is obtained. This concept is visually depicted in Figure~\ref{fig:theorems}-A, where the curve resulting from translating and rotating back the transformed curve (shown in red dashed line) aligns with the original curve. Another notable outcome from this study is that in line with the planarity theorem~\ref{theorem:planarity}, both the original and the transformed curves remain confined within the same $xz$-plane (colored yellow) as the interpolation data (represented by the black arrows).

The second study focuses on the reversion invariance property stated in theorem~\ref{theorem:reversed}. Using the same hermite data as before, we show how the optimal interpolants $\bm{\phi}^*=0$ provide the exact same curve for the reversed data. This can be visualized in Fig.~\ref{fig:theorems}-B, where the original and reversed curves only differ by the direction of the adapted-frame.

\begin{figure*}[!h]
\centering
\includegraphics[width=\textwidth]{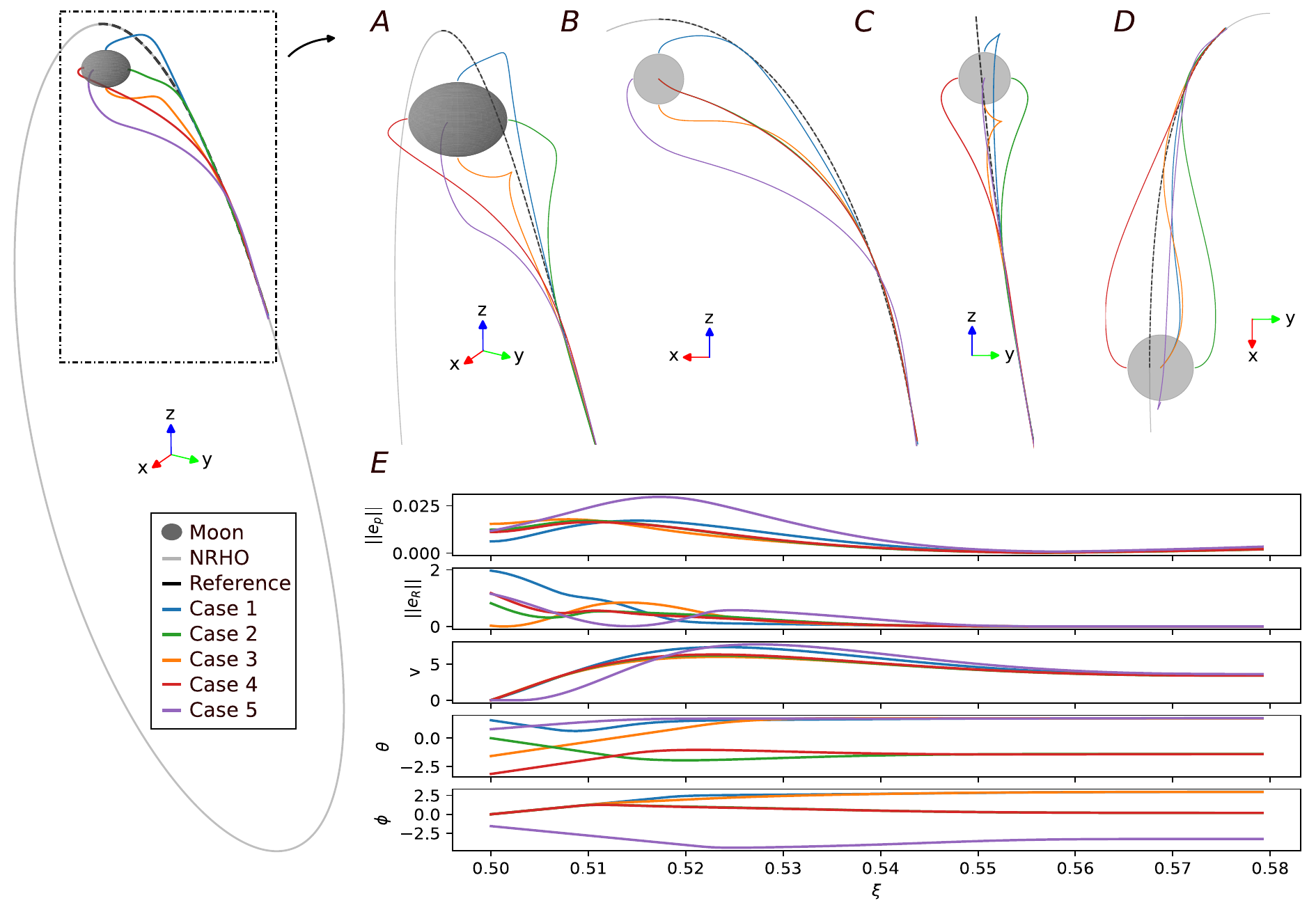}
\caption{Ascent maneuvers that take the lunar lander from five different positions at the Moon's surface back to the NRHO. The respective motions have been computed by a prediction-based controller whose dynamics evolve according to the PHODCOS coordinate system. The Moon and NRHO are given in gray, while the five lunar lander trajectories are colored blue, green, orange, red and purple. The left column provides a broad overview of the converging maneuvers towards the orbit, while panels A), B), C), and D) present more detailed isometric, front, side, and top views respectively. Panel E) illustrates the position and orientation errors between the lunar lander and the NRHO (rows 1 and 2), its velocity magnitude (row 3) and gimbal angles (rows 4 and 5), color coded according to the different initial conditions.}\label{fig:navigation}
\end{figure*}
\section{Tutorial Example: Lunar Gateway}\label{sec:example}
To illustrate the applicability of PHODCOS, we focus on the Near Rectilinear Halo Orbit (NRHO), a highly nonlinear geometric reference selected for the placement of the Lunar Gateway, a space station for the Artemis program aimed at establishing a permanent human presence on the Moon~\cite{whitley2018earth}. Unlike existing relative spacecraft maneuvering models that require the orbit to be circular or elliptical, PHODCOS does not impose constraints on shape or size, making it applicable to any arbitrary curve~\cite{sullivan2017comprehensive,willis2023convex}. 

We demonstrate the efficacy of PHODCOS by dividing the analysis into two phases. First, we use Algorithm~\ref{algo:papor} to parameterize the NRHO in both the Earth-Centered Inertial (ECI) and Earth-Moon Barycentric Rotating (EMBR) reference frames. Following that, we showcase how PHODCOS allows for the formulation of trajectory optimization algorithms, by developing a prediction-based controller for getting from the lunar surface back to orbit.

\subsection{Parameterizing the NRHO}
We obtain the NRHO data from a publicly available dataset\footnote{NRHO data available at \url{https://ssd.jpl.nasa.gov/tools/periodic_orbits.html}. The specific settings for the selected orbit are as follows: Earth-Moon system, Halo orbit family, L2 libration point, 6-8 day period range, ID 95.}. This data is fitted to a fourth-order spline, resulting in two parametric curves $\bm{\gamma}$: one in the Earth-Centered Inertial (ECI) frame and the other in the Earth-Moon Barycentric Rotating (EMBR) frame. To assign the PHODCOS coordinate system to both curves, we apply Algorithm~\ref{algo:papor} using $n_s = 256$ segments.

The resultant coordinate systems are depicted in Fig.~\ref{fig:nrho_cover} (ECI) and Fig.~\ref{fig:nrho} (EMBR). These figures confirm key characteristics of the NRHO: its high orbital eccentricity brings it very close to the lunar north pole, while it deviates significantly when passing over the south pole. This occurs approximately four times for each lunar orbit, consistent with the 9:2 orbital resonance, and results in four highly nonlinear transitions per orbital period. These transitions are clearly reflected in the parametric speed $\sigma$ and the angular velocity $\bm{\omega}$ of the coordinate system, as shown in panels D and E of Fig.~\ref{fig:nrho}. Despite these abrupt transitions, as discussed in Section~\ref{sec:phodcos}, PHODCOS ensures that all parametric functions maintain at least $C^2$ continuity, thus making them suitable for gradient-based algorithms, as will be demonstrated in the next subsection.

\subsection{Predictive control for ascent of a lunar lander}
To exemplify how PHODCOS can be leveraged in the design of optimization- or learning-based algorithms, we formulate a prediction-based control scheme based on the parameterization of the NRHO carried out in the previous subsection. More specifically, we design a trajectory optimization program that takes a lunar lander from the Moon's surface back to orbit.

In this model, the lunar lander is approximated as a gimbaled rocket with states defined as $\bm{x} = \left[\bm{p}, v, \theta, \phi\right]\in \mbb{R}^6$, where $\bm{p}\in \mbb{R}^3$ represents the position vector, and $\{v, \theta, \phi\} \in \mbb{R}$ denote the velocity magnitude and gimbal angles, respectively. The system's control inputs are the acceleration magnitude $a \in \mathbb{R}$ and the gimbal's angular rates $\{\omega_\theta, \omega_\phi\} \in \mbb{R}$, collectively forming the control vector $\bm{u} = \left[a, \omega_\theta, \omega_\phi\right]\in \mbb{R}^3$. Consequently, the motions of the lunar lander are decribed by the following equations of motion:
\begin{subequations} \label{eq:rocket}
\begin{align}
	 \bm{\dot{p}}&=v \left[\cos{\phi}\cos{\theta},\,\sin{\phi},\,\cos{\phi}\sin{\theta}\right]\\
	 \dot{v}&= a,\quad \dot{\theta} = \omega_\theta, \quad \dot{\phi} = \omega_\phi\,,
\end{align}
\end{subequations}
which can be written in the common form of nonlinear dynamical system as $\bm{\dot{x}} = f(\bm{x},\bm{u})$. By designating the PHODCOS coordinate system to the NRHO --as done in the previous subsection--, we can opt for a spatial parameterization, whereby the system dynamics expressed in~\eqref{eq:rocket} evolve according to the path parameter $\xi$, instead of the conventional time-based representation. To make this distinction clear and following the notation in Section~\ref{sec:preliminaries}, we rewrite the equations of motion as $\bm{{x}}^{'}(\xi) = f(\bm{x}(\xi),\bm{u}(\xi))$, where $(\cdot)' = \dv{(\cdot)}{\xi}$. Hence, rather than having a temporal horizon $T$, we have a spatial look-ahead $\Xi=\xi_f-\xi_0$. This implies that our prediction horizon pertains to a specific segment of the orbit, regardless of time.

Throughout this prediction interval, our objective is to find a trajectory to ascend from a given state $\bm{x_0}$ on the lunar surface back to the NRHO, while ensuring that the motion is dynamically compliant with~\eqref{eq:rocket} and adheres to all state and input limits, i.e., $\bm{x}\in\mathcal{X}$, $\bm{u}\in\mathcal{U}$. To compute such a trajectory, we solve the subsequent Optimal Control Problem (OCP):
\begin{subequations}\label{eq:OCP}
    \begin{flalign}
     &\min_{\bm{x}(\cdot),\bm{u}(\cdot)} \int_{\xi_0}^{\xi_f}||\bm{p}(\xi) - \bm{p}_\text{NRHO}(\xi)||^2_Q + ||\bm{u}(\xi)||^2_R\; d\xi&
    \end{flalign}
    \vspace{-5mm}
	\begin{alignat}{3}
	\text{s.t.}\quad& \bm{x}(\xi_0) = \bm{x_0}\,,\\
	&\bm{x}'(\xi) =f(\bm{x}(\xi),\bm{u}(\xi)), &\quad&\xi \in \left[\xi_0,\xi_f\right]\label{eq:dynamic_const}\\
	&\bm{x}(\xi)\in\mcl{X}\,,\,\bm{u}(\xi)\in\mcl{U}\,,    &\quad&\xi \in \left[\xi_0,\xi_f\right]\,.
 \label{eq:mpc_spatial_cond_cont}
	\end{alignat}
\end{subequations}
where $\bm{p}_\text{NRHO}(\xi)$ is readily available from the parametric functions obtained in the previous subsection by applying the PHODCOS Algorithm~\ref{algo:papor} to the NRHO.


OCP~\eqref{eq:OCP} presents a basic example, as it only begins to explore the possibilities offered by parameterizing trajectory optimization problems using curvilinear differentiable coordinate systems like PHODCOS.  This formulation can be extended to handle more complex behaviors, such as explicit control over the progress variable, enhancing robustness against disturbances, shaping the transverse distance to the orbit, or allocating protective funnels around the orbit~\cite{arrizabalaga2022spatial, arrizabalaga2023pose, arrizabalaga2024differentiable}. For a comprehensive discussion of trajectory optimization strategies and applications beyond navigation, see \cite{arrizabalaga2024universal}. Nevertheless, to maintain a clear and precise perspective, OCP~\eqref{eq:OCP} represents the simplest formulation, adequately showing that PHODCOS can be integrated into optimization- and learning-based algorithms, which require differentiating over parametric functions $\{\bm{p}_\text{NRHO}(\xi),\text{R}_\text{NRHO}(\xi),\bm{\omega}_\text{NRHO}(\xi)\}$ by retrieving first- and second-order derivative information.


We approximate OCP \eqref{eq:OCP} as a nonlinear program according to the multiple-shooting method \cite{bock1984multiple} in which the optimization horizon $\Xi=0.08$ is split into $N=100$ segments, each with constant decision variables. To ensure the validity of the method, we consider five distinct initial locations on the lunar surface: the north and south poles, both extremes along the equator, and the region directly opposite the NRHO horizon segment.

The resulting trajectories, along with the Moon and the NRHO, are shown in Fig.~\ref{fig:navigation}. The segment of the orbit that falls under the prediction horizon and serves as a reference $\bm{p}_\text{NRHO}(\xi)$ is depicted by the dashed black line. The motions of the lunar lander are given in the left column of the figure, as well as in panels A), B), C) and D). These are colored according to the initial conditions, i.e., blue and orange for the north and south poles, green and red for the extremes along the equator, and purple for the region opposite to the horizon segment. To gain a better understanding of the convergence of these trajectories, panel E) shows the position $\bm{e}_p$ and orientation $\bm{e}_R$ errors associated to all five trajectories, as well as their respective velocity magnitudes $v$ and gimbal angles $\theta, \phi$. From here, it is apparent that all five cases succeed in reaching the orbit within the prediction horizon.


\section{Conclusions}\label{sec:conclusions}
In this work, we proposed PHODCOS, an algorithm to assign a differentiable coordinate system to any given curve. For this purpose, we rely on a Pythagorean Hodograph (PH) curve-based hermite interpolation scheme that parameterizes any arbitrary curve into a geometric reference with a moving frame associated with it. The conversion algorithm is characterized by three attractive properties: efficiency, bounded conversion error, and differentiability. First, all operations are provided in closed form, eliminating the need for iterative or numerical procedures. This makes PHODCOS computationally lightweight and suitable for real-time applications. Second, we have proved analytically and numerically that the algorithm achieves a sixth-order approximation, ensuring that the deviation between the coordinate system and the curve remains within a predefined tolerance. Third, the parametric functions produced by PHODCOS are at least two times differentiable, making them compatible with gradient-based numerical routines. To demonstrate the practical utility of PHODCOS, we have applied it to the Near-Rectilinear Halo Orbit (NRHO) and formulated a trajectory optimization problem to guide a lunar lander from the Moon's surface back to the NRHO. To facilitate broader use in science and engineering, we have made PHODCOS available as an \href{https://github.com/jonarriza96/phodcos}{open source} software library, along with the examples presented.


\appendices{}              

\section{Derivation of preimage control-point $\mcl{A}_4$}
\begin{align}\label{eq:cp}
\begin{split}
\bm{c}_{\bm{p}} &= 
\frac{1}{112633092}(147807 \AiA{0}{0}+72270 \AiA{0}{1}+\\
&\quad30954 \AiA{0}{2}+9702 \AiA{0}{3}-3234 \AiA{0}{5}-\\
&\quad3150 \AiA{0}{6}-1818 \AiA{0}{7}-565 \AiA{0}{8}+\\
&\quad722{7}{0}\AiA{1}{0}+72732 \AiA{1}{1}+47124 \AiA{1}{2}+\\
&\quad19404 \AiA{1}{3}-8820 \AiA{1}{5}-9324 \AiA{1}{6}-\\
&\quad5668 \AiA{1}{7}-1818 \AiA{1}{8}+30954 \AiA{2}{0}+\\
&\quad47124\AiA{2}{1}+40572 \AiA{2}{2}+20580 \AiA{2}{3}-\\
&\quad12348 \AiA{2}{5}-14308 \AiA{2}{6}-9324 \AiA{2}{7}-\\
&\quad3150 \AiA{2}{8}+9702 \AiA{3}{0}+19404 \AiA{3}{1}+\\
&\quad20580\AiA{3}{2}+12348 \AiA{3}{3}-9604 \AiA{3}{5}-\\
&\quad12348 \AiA{3}{6}-8820 \AiA{3}{7}-3234 \AiA{3}{8}-\\
&\quad3234 \AiA{5}{0}-8820 \AiA{5}{1}-12348 \AiA{5}{2}-\\
&\quad9604\AiA{5}{3}+12348 \AiA{5}{5}+20580 \AiA{5}{6}+\\
&\quad19404 \AiA{5}{7}+9702 \AiA{5}{8}-3150 \AiA{6}{0}-\\
&\quad9324 \AiA{6}{1}-14308 \AiA{6}{2}-12348 \AiA{6}{3}+\\
&\quad20580\AiA{6}{5}+40572 \AiA{6}{6}+47124 \AiA{6}{7}+\\
&\quad30954 \AiA{6}{8}-1818 \AiA{7}{0}-5668 \AiA{7}{1}-\\
&\quad9324 \AiA{7}{2}-8820 \AiA{7}{3}+19404 \AiA{7}{5}+\\
&\quad47124\AiA{7}{6}+72732 \AiA{7}{7}+72270 \AiA{7}{8}-\\
&\quad565 \AiA{8}{0}-1818 \AiA{8}{1}-3150 \AiA{8}{2}-\\
&\quad3234 \AiA{8}{3}+9702 \AiA{8}{5}+30954 \AiA{8}{6}+\\
&\quad72270\AiA{8}{7}+147807 \AiA{8}{8})\,.
\end{split}
\end{align}
Using the commutative operation declared in Definition~\ref{definition:comm_operation}, eq.~\eqref{eq:cp} can be simplified into~\eqref{eq:cp_star}.

\subsection{More Comments}
Use section and subsection keywords as usual.

\section{Theorems for $\mathbf{\phi^*}=\mathbf{0}$}\label{appendix:collorary_theorems}
\begin{theorem}[\bfseries Planarity preservation]\label{theorem:planarity}
Given $C^4$ hermite data confined to a plane $P \subset \mbb{R}^3$, the interpolant $\bm{p_{\phi^*}}$ remains within $P$.
\end{theorem}
\begin{proof}
See Theorem 3.11 in~\cite{vsir2007}.
\end{proof}
\begin{theorem}[\bfseries Invariance of interpolants]\label{theorem:invariance}
The parameterization of solutions, as described in Definition~\ref{definition:standard_form}, remains unchanged under rigid body motions, i.e., special orthogonal transformations). However, reflections invert the signs of all parameters, meaning $\bm{\phi}\rightarrow-\bm{\phi}$. Hence, $\bm{\phi}^*=\bm{0}$ exhibits invariance under all orthogonal transformations.
\end{theorem}
\begin{proof}
See Lemma 3.6 in~\cite{vsir2007}
\end{proof}
\begin{theorem}[\bfseries Reversion invariance]\label{theorem:reversed}
Define $\bm{p}_\phi(\xi)$ as the interpolants of the dataset $\{\bm{p}_b$, $\bm{p}_e$, $\bm{v}_b$, $\bm{v}_e$, $\bm{a}_b$, $\bm{a}_e$, $\bm{j}_b$, $\bm{j}_e$, $\bm{s}_b$, $\bm{s}_e\}$. Similarly, let $\bar{\bm{p}}_\phi(\xi)$ represent the interpolants of the corresponding "reversed" dataset $\{\bar{\bm{p}}_b = \bm{p}_e$, $\bar{\bm{v}}_b = \bm{v}_e$, $\bar{\bm{a}}_b = \bm{a}_e$, $\bar{\bm{j}}_b = \bm{j}_e$, $\bar{\bm{s}}_b = \bm{s}_e\}$ and $\{\bar{\bm{p}}_e = \bm{p}_b$, $\bar{\bm{v}}_e = \bm{v}_b$, $\bar{\bm{a}}_e = \bm{a}_b$, $\bar{\bm{j}}_e = \bm{j}_b$, $\bar{\bm{s}}_e = \bm{s}_b\}$. Consequently, we have that $\bar{\bm{p}}_{\bm{\phi}}(1-\xi) = \bm{p}_{-\phi}(\xi)$, leading to $\bm{p}_{\bm{\phi^*}} = \bar{\bm{p}}_{\bm{\phi^*}}$.
\end{theorem}
\begin{proof}
See Theorem 3.10 in~\cite{vsir2007}.
\end{proof}
\section{Derivation of hodograph control-points}\label{appendix:h_ctrlpts}
Focusing the specific case of $C^4$ interpolation and a preimage $\mcl{A}(\xi)$ of order 8, the extended Bernstein Bezier form in~\eqref{eq:hodograph} is
\begin{align}
\begin{split}
    \mcl{A}(\xi) &= \xi^8\mcl{A}_8+8(1-\xi)\xi^7\mcl{A}_7+28(1-\xi)^2\xi^6\mcl{A}_6 +\\
    &\quad56(1-\xi)^3\xi^5\mcl{A}_5 + 70(1-\xi)^4\xi^4\mcl{A}_4+\\
    &\quad56(1-\xi)^5\xi^3\mcl{A}_3+28(1-\xi)^6\xi^2\mcl{A}_2+\\
    &\quad8(1-\xi)^7\xi\mcl{A}_1+(1-\xi)^8\mcl{A}_0\,.
\end{split}
\end{align}
Leveraging the quadratic expression in~\eqref{eq:hodograph}, the control-points of the hodograph $\bm{h}(\xi)$ are defined as:
\begin{subequations}\label{eq:h_ctrlpts_original}
\allowdisplaybreaks
\begin{align}
\bm{h}_0 &= \AiA{0}{0}\,,\\
\bm{h}_1 &= \frac{1}{16}(8\AiA{0}{1}+ 8\AiA{1}{0})\,,\\
\bm{h}_2 &= \frac{1}{120}(28\AiA{0}{2}+64\AiA{1}{1}+28\AiA{2}{0})\,,\\
\begin{split}
\bm{h}_3 &= \frac{1}{560}(56\AiA{0}{3}+224\AiA{1}{2}+224\AiA{2}{1}+\\
         &\quad56\AiA{3}{0})\,,
\end{split}\\
\begin{split}
\bm{h}_4 &= \frac{1}{1820}(70\AiA{0}{4}+448\AiA{1}{3}+784\AiA{2}{2}+\\
         &\quad448\AiA{3}{1}+70\AiA{4}{0})\,,
\end{split}\\
\begin{split}
\bm{h}_5 &=\frac{1}{4368}(56\AiA{0}{5}+560\AiA{1}{4}+1568\AiA{2}{3}+\\
            &\quad1568\AiA{3}{2}+560\AiA{4}{1}+ 56\AiA{5}{0})\,,
\end{split}\\
\begin{split}
\bm{h}_6 &= \frac{1}{8008}(28\AiA{0}{6}+448\AiA{1}{5}+1960\AiA{2}{4}+\\
            &\quad3136\AiA{3}{3}+1960\AiA{4}{2}+448\AiA{5}{1}+\\
            &\quad28\AiA{6}{0})\,,
\end{split}\\
\begin{split}
\bm{h}_7 &= \frac{1}{11440}(8\AiA{0}{7}+224\AiA{1}{6}+1568\AiA{2}{5}+\\
            &\quad3920\AiA{3}{4}+3920\AiA{4}{3}+1568\AiA{5}{2}+\\
            &\quad224\AiA{6}{1}+8\AiA{7}{0})\,,
\end{split}\\
\begin{split}
\bm{h}_8 &= \frac{1}{12870}(\AiA{0}{8}+64\AiA{1}{7}+784\AiA{2}{6}+\\
            &\quad3136\AiA{3}{5}+4900\AiA{4}{4}+3136\AiA{5}{3}+\\
            &\quad784\AiA{6}{2}+64\AiA{7}{1}+\AiA{8}{0})\,,
\end{split}\\
\begin{split}
\bm{h}_9 &= \frac{1}{11440}(8\AiA{1}{8}+224\AiA{2}{7}+1568\AiA{3}{6}+\\
            &\quad3920\AiA{4}{5}+3920\AiA{5}{4}+1568\AiA{6}{3}+\\
            &\quad224\AiA{7}{2}+8\AiA{8}{1})\,,
\end{split}\\
\begin{split}
\bm{h}_{10} &= \frac{1}{8008}(28\AiA{2}{8}+448\AiA{3}{7}+1960\AiA{4}{6}+\\
            &\quad 3136\AiA{5}{5}+1960\AiA{6}{4}+448\AiA{7}{3}+\\
            &\quad28\AiA{8}{2})\,,
\end{split}\\
\begin{split}
\bm{h}_{11} &= \frac{1}{4368}(56\AiA{3}{8}+560\AiA{4}{7}+1568\AiA{5}{6}+\\
            &\quad 1568\AiA{6}{5}+560\AiA{7}{4}+56\AiA{8}{3})\,,
\end{split}\\
\begin{split}
\bm{h}_{12} &= \frac{1}{1820}(70\AiA{4}{8}+448\AiA{5}{7}+784\AiA{6}{6}+\\
            &\quad 448\AiA{7}{5}+70\AiA{8}{4})\,,
\end{split}\\
\bm{h}_{13} &= \frac{1}{560}(56\AiA{5}{8}+224\AiA{6}{7}+224\AiA{7}{6}+\\
            &\quad56\AiA{8}{5})\,,\\
\bm{h}_{14} &= \frac{1}{120}(28\AiA{6}{8}+64\AiA{7}{7}+28\AiA{8}{6})\,,\\
\bm{h}_{15} &= \frac{1}{16}(8\AiA{7}{8}+8\AiA{8}{7})\,,\\
\bm{h}_{16} &= \AiA{8}{8}\,.
\end{align}
\end{subequations}
Using the commutative operation declared in Definition~\ref{definition:comm_operation}, these equations can be simplified into~\eqref{eq:h_ctrlpts}.
This is the second appendix.


\bibliographystyle{IEEEtran}
\bibliography{lib}




\thebiography
\begin{biographywithpic}
{Jon Arrizabalaga}{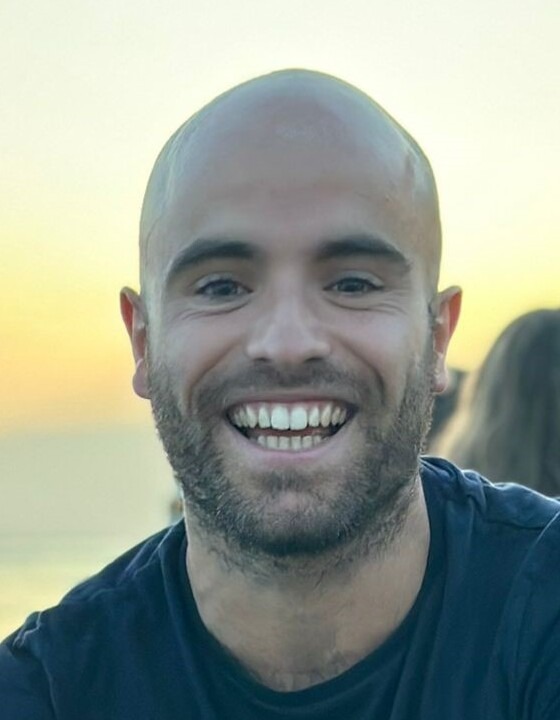}
is pursuing his PhD at the Munich Institute of Robotics and Machine Intelligence, Technical University of Munich, and is a visiting researcher at the Robotics Institute of Carnegie Mellon University. He received a BSc. in mechanical engineering from the University of Navarre (Spain, 2018) and an MSc. in mechatronics and robotics at KTH Royal Institute of Technology (Sweden, 2020). His research focuses on the convergence of optimization, control, and geometry, particularly in their applications to robotics and autonomous systems.
\end{biographywithpic}
\begin{biographywithpic}
{Fausto Vega}{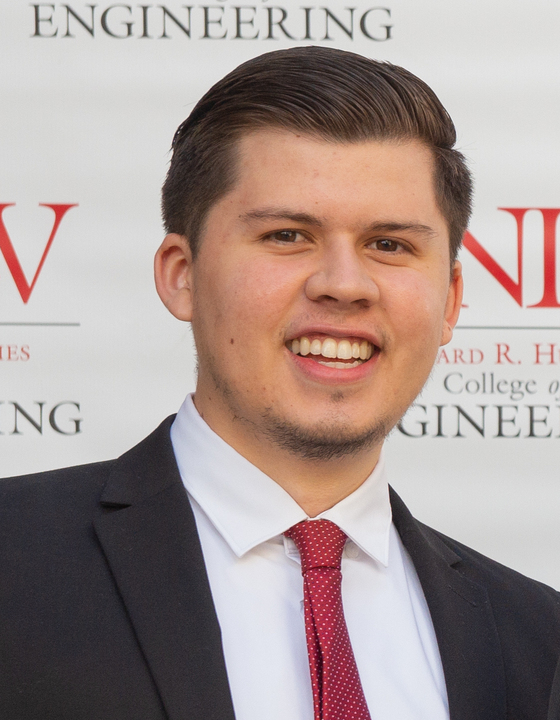}
is a PhD student at the Robotics Institute at Carnegie Mellon University (CMU). He received a BS in mechanical engineering from the University of Nevada, Las Vegas and an MS in Robotics at CMU. His research interests include optimal control for maneuver design small spacecraft state estimation.
\end{biographywithpic}
\begin{biographywithpic}
{Zbyněk Šír}{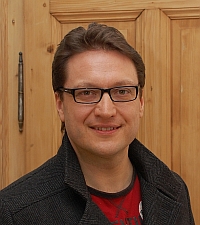}
is an associate professor at the Mathematical Institute in the Faculty of Mathematics and Physics of Charles University in Prague. He received a DEA (Diplome d’études approfondies) from University Pierre et Marie Curie, a Master with specelization in Riemannian Geometry and Theory of Representations, and a PhD in mathematics with specialization in history of French geometry, both from Charles University of Prague. His research interestes include CAGD, theoretical differential geometry, history of geometry and other applied geometric fields.
\end{biographywithpic}
\begin{biographywithpic}
{Zachary Manchester}{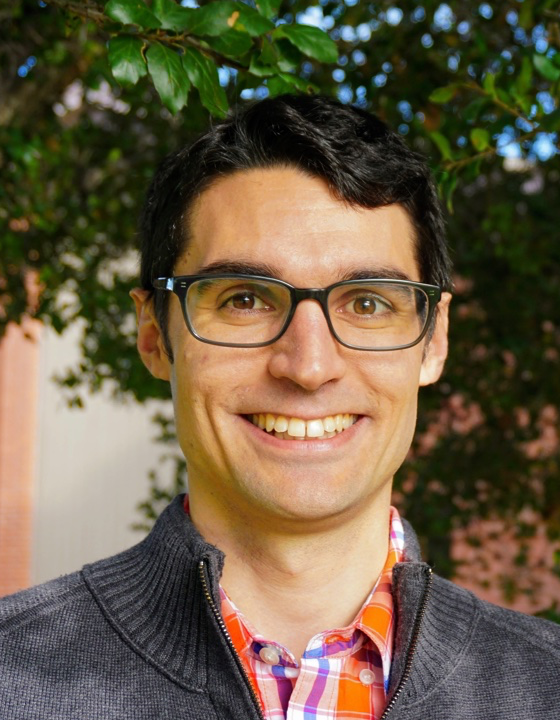}
is an assistant professor in the Robotics Institute at Carnegie Mellon University and founder of the Robotic Exploration Lab. He received a PhD in aerospace engineering in 2015 and a BS in applied physics in 2009, both from Cornell University. His research interests include control and optimization with applications to aerospace and robotic systems.
\end{biographywithpic} 
\begin{biographywithpic}
{Markus Ryll}{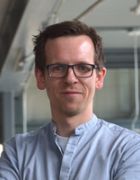}
is an assistant professor in the Department of Aerospace and Geodesy at the Technical University of Munich and head of the Autonomous Aerial Systems Lab. He received a PhD in 2014 from Max Planck Institute for Biological Cybernetics. Between 2014 and 2017, he was a postdoctoral researcher at the Laboratory for Analysis and Architecture of Systems (LAAS-CNRS). From 2018 to 2020, he continued his postdoctoral work in the Robust Robotics Group at the Massachusetts Institute of Technology (MIT, CSAIL). His research focuses on enabling autonomous aerial robots to interact with real-world environments.
\end{biographywithpic} 
\end{document}